\pdfoutput=1
\documentclass[twoside]{article}

\usepackage[preprint]{aistats2027}

\usepackage[round]{natbib}

\usepackage[hyphens]{url}
\usepackage{graphicx}
\usepackage{caption}
\usepackage{subcaption}
\usepackage{float}       %
\usepackage{algorithm}
\usepackage{algorithmic}
\usepackage{booktabs}

\usepackage{amsmath}
\usepackage{amssymb}
\usepackage{mathtools}
\usepackage{amsthm}

\usepackage{xcolor}
\usepackage{makecell}
\usepackage[shortlabels]{enumitem}
\usepackage{wasysym}

\newcommand{\X}{\mathbf{X}}
\newcommand{\Z}{\mathbf{Z}}
\newcommand{\bPhi}{\mathbf{\Phi}}
\newcommand{\I}{\mathbf{I}}
\newcommand{\Y}{\mathbf{Y}}
\newcommand{\x}{\mathbf{x}}
\newcommand{\W}{\mathbf{W}}
\newcommand{\A}{\mathbf{A}}
\newcommand{\bv}{\mathbf{b}}
\newcommand{\Si}{\boldsymbol{\Sigma}}
\newcommand{\muvec}{\boldsymbol{\mu}}
\newtheorem{theorem}{Theorem}
\newtheorem{proposition}[theorem]{Proposition}

\begin{document}

\runningtitle{Time-Correlated Video Bridge Matching}

\runningauthor{Vasilev, Ivanov, Gushchin, Kovaleva, Korotin}

\twocolumn[

\aistatstitle{Time-Correlated Video Bridge Matching}

\aistatsauthor{ Viacheslav Vasilev$^{1}$ \And Arseny Ivanov$^{2,3,4}$ \And Nikita Gushchin$^{3,2}$ \And Maria Kovaleva$^{1}$ \And Alexander Korotin$^{3,2}$ }

\aistatsaddress{ $^1$Kandinsky Lab \quad $^2$AXXX \quad $^3$Applied AI Institute \quad $^4$HSE University \\
\texttt{viacheslav.vasilev@kandinskylab.ai} } ]

\begin{abstract}
Bridge Matching (BM) learns a stochastic transport between two data distributions and is therefore a natural tool for data-to-data tasks such as video frame interpolation, image-to-video generation, and video super-resolution. Existing bridge models for video, however, use a reference process that treats the frames of a clip independently, so temporal structure enters only through the network or the choice of the bridge endpoints. We propose Time-Correlated Video Bridge Matching (TCVBM), which places temporal structure into the reference process itself: the frames of a clip are coupled by a linear SDE whose drift is a discrete Laplacian along the time axis, and known frames enter as boundary conditions. The transition and bridge distributions remain Gaussian and available in closed form, so training and sampling are as simple as in standard BM, which is recovered in the limit of zero coupling. A spectral analysis shows that the prior damps high temporal frequencies, injects temporally smooth noise, and relaxes towards a boundary-determined profile, such as the linear interpolation between key frames. In controlled comparisons with frame-independent bridges under identical architectures and training budgets, TCVBM improves most automatic metrics for frame interpolation (Moving MNIST, UCF-101, Vimeo-90K) and image-to-video generation (Moving MNIST, UCF-101), e.g., reducing FVD from 688 to 611 for frame interpolation on UCF-101, and human raters prefer its frame interpolations for temporal consistency in 49\% of comparisons versus 15\% for the baseline ($p<10^{-17}$); for super-resolution, it performs on par with BM on automatic metrics.
\end{abstract}

\section{Introduction}

Diffusion models~\citep{pmlr-v37-sohl-dickstein15, 10.5555/3495724.3496298, song2021scorebased} have emerged as a powerful paradigm for generative modeling, achieving remarkable results in high-fidelity data synthesis~\citep{saharia2022photorealistictexttoimagediffusionmodels, rombach2022highresolutionimagesynthesislatent, vladimir-etal-2024-kandinsky, flux2024, arkhipkin2025kandinsky50familyfoundation}. By iteratively denoising samples from a Gaussian distribution, these models produce diverse and realistic outputs. In data-to-data tasks, they are typically adapted through conditioning, but generation still starts from an uninformative Gaussian prior.

In contrast, Bridge Matching (BM) constructs a stochastic transport directly between two data distributions~\citep{peluchetti2023non, peluchetti2023diffusion, liu2022let, zhoudenoising}, so that generation starts from the corrupted data itself. These methods demonstrate strong performance in image-to-image tasks~\citep{shi2023diffusion, i2sb}, and bridge models have recently been applied to video frame interpolation~\citep{lyu2024cbbd, lyu2025tlbvfi} and image-to-video generation~\citep{wang2025framebridge}. However, in these models the reference (prior) process perturbs the frames independently: temporal structure enters through the network and the endpoints of the bridge (e.g., neighboring frames or a static video), but not through a coupling of the frames in the reference dynamics.

Placing temporal structure in the prior has three practical advantages. First, it adds little cost: the bridge stays Gaussian and in closed form, adds no trainable parameters, and requires only a few matrix products over the frame dimension per sampling step. Second, it is architecture-agnostic: it changes only how training pairs are corrupted and how samples are drawn, so it can be combined with any network, including pretrained video models, as we show by fine-tuning CogVideoX-2B. Third, it is interpretable: known frames enter the SDE as boundary conditions, so that for frame interpolation with $\alpha=1$ the prior mean relaxes towards the linear interpolation between the key frames, and $\alpha\to0$ recovers standard BM.

\textbf{Contribution.} We propose Time-Correlated Video Bridge Matching (TCVBM), which places temporal structure into the reference process of the bridge (Figure~\ref{fig:main_scheme}). Our contributions are as follows:

\begin{enumerate}
    \item We formulate Time-Correlated Video Bridge Matching, in which a video is a structured sample whose frames are coupled by a linear SDE prior with a discrete temporal Laplacian drift; standard BM is recovered in the limit $\alpha\to0$ (Section~\ref{section:method});

    \item We give the closed-form transition, score and bridge distributions of this prior, task-specific boundary conditions for frame interpolation, image-to-video generation, and video super-resolution, and a spectral analysis showing that the prior damps high temporal frequencies, injects temporally smooth noise, and relaxes towards a boundary-determined profile (Sections~\ref{sec:prior_theory}--\ref{sec:spectral});

    \item We show, both in a controlled setting where all methods share the same small 2D U-Net (Moving MNIST) and in combination with pretrained and standard video architectures (UCF-101, Vimeo-90K), that the time-correlated prior improves frame interpolation and image-to-video generation over frame-independent bridges, with a significant human preference for frame interpolation (Sections~\ref{sec:simple_experiments}--\ref{sec:results}).
\end{enumerate}

\section{Related Work}

\subsection{Bridge Models}\label{sec:related_dbm}

Like flow matching~\citep{lipman2023flow}, Bridge Matching learns a transport along pre-defined paths, but between two data distributions, so that generation starts from data; it performs well in image restoration, translation, and reconstruction~\citep{delbracio2024inversiondirectiterationalternative, zhoudenoising, i2sb}. Bridges with a non-zero drift in the reference process, e.g., Ornstein--Uhlenbeck bridges~\citep{yue2024goub}, show that the choice of the reference process matters. For video, Brownian bridge diffusion has been used for frame interpolation~\citep{lyu2024cbbd, lyu2025tlbvfi}, and FrameBridge~\citep{wang2025framebridge} formulates image-to-video generation as a bridge from a static video to the target video. Bridge and Schr\"odinger-bridge models have also been applied to time series imputation and generation~\citep{chen2023csbi, park2024timebridge}, e.g., by designing data-dependent endpoint priors. In the video bridge models above, the reference process acts on every frame (or latent) independently. Our approach instead couples the frames in the drift of the reference SDE, which explicitly models linear correlations between video frames and contains the frame-independent Brownian bridge as the limit $\alpha\to0$.

\subsection{Temporal Modeling for Video Data}\label{sec:related_temporal}

\paragraph{Architectural Approach.} Video generation models often adapt pretrained image models by adding temporal modules such as 3D convolutions or temporal attention layers~\citep{ho2022imagenvideohighdefinition, blattmann2023videoldm, arkhipkin2023fusionframesefficientarchitecturalaspects}, and video diffusion transformers follow the same principle~\citep{ma2025latte}. A key challenge is the computational cost of attention, addressed by sparse attention and adaptive masking~\citep{zhang2025fast, xi2025sparsevideogenacceleratingvideo, mikhailov2025nablanablaneighborhoodadaptiveblocklevel}. These works do not alter the forward (reference) process and are complementary to our approach.

\paragraph{Prior Modification.} Other works incorporate correlation between video frames into the noise of diffusion models~\citep{ge2024preservecorrelationnoiseprior, chang2025iwarpednoisetemporallycorrelated, liu2025equivariancefastsamplingvideo}. Blurring and heat-dissipation diffusion~\citep{rissanen2023ihdm, hoogeboom2023blurring} use a Laplacian in the forward process to damp high spatial frequencies of images. TCVBM combines these ideas in a different place: a temporal Laplacian in the reference SDE of a bridge between data distributions, which simultaneously yields temporally correlated noise and a temporally smoothed mean (Section~\ref{sec:spectral}).

Additional related works on specific video-to-video tasks can be found in Appendix~\ref{sec:additional_related}.

\section{Background on Bridge Matching}\label{sec:background}

Bridge Matching~\citep{peluchetti2023non, peluchetti2023diffusion, liu2022let, shi2023diffusion} constructs diffusion processes that translate between clean data \( p(\x_0) \) and corrupted data \( p(\x_T) \) on \( \mathbb{R}^D \) while incorporating prior dynamics. Consider a coupling \( p(\x_0, \x_T) = p(\x_0)\,p(\x_T \mid \x_0) \), and let the prior process be defined by the stochastic differential equation (SDE):
\begin{equation}
    d\x_t = f(\x_t, t)\,dt + g(t)\,d\W_t, \label{eq:prior}
\end{equation}
where \( f(\x_t, t) \) is a drift function, \( g(t) \) is a time-dependent noise scale, and \( \W_t \) is a standard Wiener process. For a fixed starting point \( \x_s \), we denote the marginal of the prior process at time \( t \) by \( q(\x_t \mid \x_s) \).

\paragraph{Bridge Distribution.} Given a pair \( (\x_0, \x_{t'}) \) from the prior, the posterior distribution $q(\x_t \mid \x_0, \x_{t'})$ of the process at time \( t < t' \) is referred to as the \emph{bridge distribution}; it follows from Bayes' rule.

\paragraph{Bridge Matching Dynamics.}
Bridge Matching aims to construct a stochastic process that interpolates between \( \x_T \) and \( \x_0 \) using a reverse-time SDE:
\[
d\x_t = \left\{ f(\x_t, t) - g^2(t)\,v^*(\x_t, t) \right\} dt + g(t)\, d\bar{\W}_t,
\]
where \( \bar{\W}_t \) is a standard Wiener process under time reversal \( t \leftarrow T - t \), \( dt \) denotes a negative infinitesimal timestep, and $v^*(\x, t) = \mathbb{E}\big[\nabla_{\x_t}\log q(\x_t\mid\x_0)\,\big|\,\x_t = \x\big]$. With this choice, the marginals of the reverse SDE started from $p(\x_T)$ match those of the mixture of bridges $q(\x_t\mid\x_0,\x_T)$ under the coupling~\citep[Sec.~3.4]{peluchetti2023diffusion}, see also~\citet[Prop.~9]{shi2023diffusion}.

\paragraph{Learning Objective.}
The score function \( v^*(\x_t, t) \) is approximated using the following optimization objective:
\begin{equation}
\label{eq:background-obj}
    \min_{\phi} \mathbb{E}_{\x_0, \x_T, t} \left[ \left\| v_\phi(\x_t, t) - \nabla_{\x_t} \log q(\x_t \mid \x_0) \right\|^2 \right],
\end{equation}
where \( \x_0 \sim p(\x_0) \), \( \x_T \sim p(\x_T \mid \x_0) \), and \( \x_t \sim q(\x_t \mid \x_0, \x_T) \) is sampled from the bridge distribution. Time \( t \) is sampled uniformly from the interval \( [0, T] \).

\section{Method}\label{section:method}

\begin{figure*}[t]
  \centering
    \includegraphics[width=0.62\linewidth]{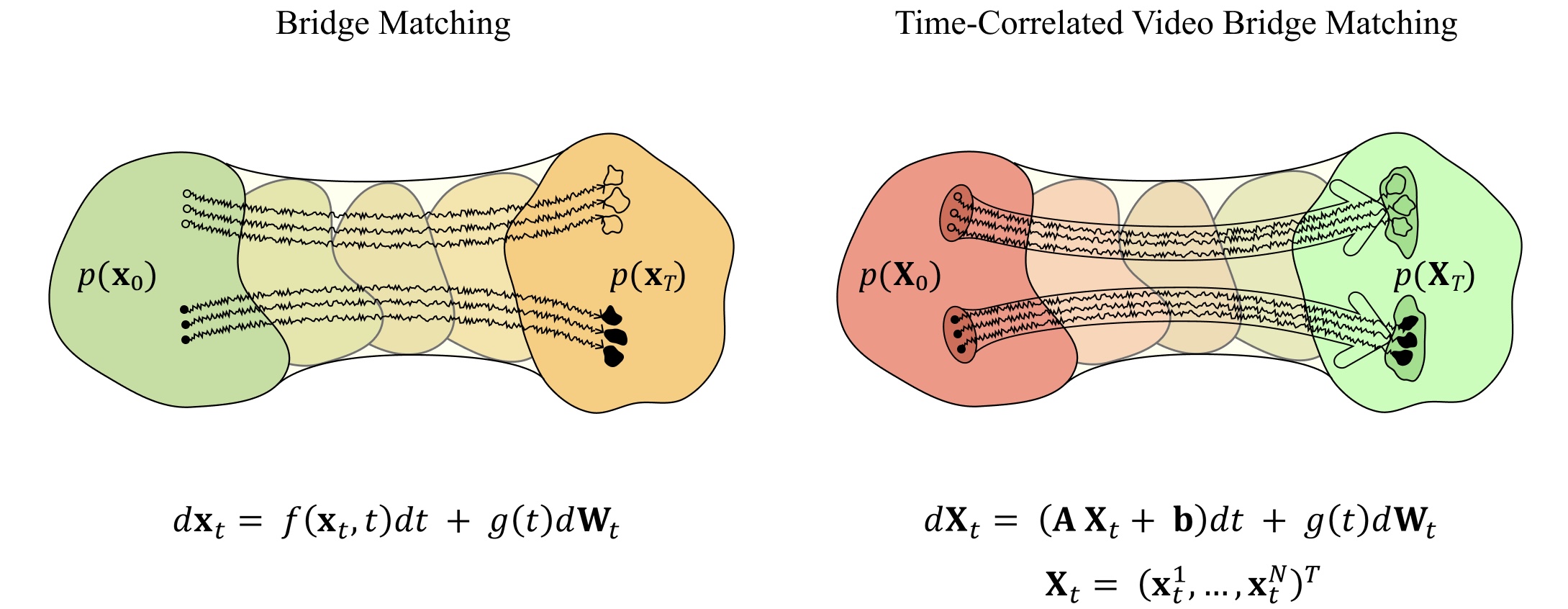}
    \caption[]{Bridge Matching (left) vs.\ TCVBM (right). Dots of the same color are frames of the same video. In BM, the reference bridge perturbs every frame independently, and frames interact only through the network. TCVBM treats a clip $\X_0$ as one sample: the drift $\A\X_t$ couples neighboring frames, the known frames enter through $\bv$, and generation runs from $\X_T$ to $\X_0$ by sampling the bridge of Proposition~\ref{thm:correlated-bridge-distribution}.}\label{fig:main_scheme}
\end{figure*}

TCVBM incorporates temporal correlations directly into the reference process of the bridge (Figure~\ref{fig:main_scheme}). The closed-form expressions below instantiate standard results for linear SDEs~\citep{sarkka2019applied} and Gaussian conditioning; proofs are given in Appendix~\ref{appendix:proofs}.

\subsection{Time-Correlated Prior Process}\label{sec:prior_theory}

We consider sequences of length \( N \), represented as
\begin{equation}\label{eq:interframes}
    {\X} = (\x^1, \ldots, \x^N),
\end{equation}
where each \( \x^n \in \mathbb{R}^D \) for \( n = 1, \ldots, N \), i.e., the rows of $\X \in \mathbb{R}^{N\times D}$ are frames and the columns are features.
We aim to define a prior diffusion process that imposes an inductive bias toward temporal smoothness across frames.

\paragraph{Prior process.} We assume that the $D$ feature dimensions evolve independently with shared temporal dynamics: each column $\x_t^{(d)} \in \mathbb{R}^N$ of $\X_t$ follows $d\x_t^{(d)} = ( \A \x_t^{(d)} + \bv^{(d)} ) dt + g(t)\, d\W_t^{(d)}$, where $\A \in \mathbb{R}^{N \times N}$ is a symmetric, invertible matrix encoding temporal correlations and $\bv^{(d)} \in \mathbb{R}^N$ is a boundary term. Equivalently, $d\X_t = (\A \X_t + \bv)\, dt + g(t)\, d\W_t$ with $\bv = [\bv^{(1)} \ldots \bv^{(D)}]$ and a matrix $\W_t$ of independent Wiener processes; all covariance and score expressions act column-wise. We use $g(t) = \sqrt{\epsilon}$ and $T=1$ throughout.

\begin{proposition}[Transition Distribution of the Time-Correlated Prior]\label{thm:correlated-process-score}
If $\X_t$ follows the linear SDE
\begin{equation}\label{eq:tc-prior}
d\X_t = (\A \X_t + \bv)\,dt + \sqrt{\epsilon}\, d\W_t, \quad \X_0 \sim \delta_{\X_0},
\end{equation}
then the marginal distribution of $\X_t$ is Gaussian:
\begin{equation}
q(\X_t|\X_0) = \mathcal{N}(\X_t \mid \muvec_{t|0}(\X_0), \Si_{t|0}),
\end{equation}
with
\begin{align}
\muvec_{t|0}(\X_0) &= e^{\A t} \X_0 + \left(e^{\A t} - \I\right)\A^{-1}\bv, \label{eq:main-mean}\\
\Si_{t|0} &= \epsilon\, \frac{e^{2\A t} - \I}{2} \A^{-1}. \label{eq:main-cov}
\end{align}
The score function is then given by
\begin{equation}
\nabla_{\X_t} \log q(\X_t | \X_0) = -\Si_{t|0}^{-1} (\X_t - \muvec_{t|0}(\X_0)).
\end{equation}
\end{proposition}
To perform bridge matching, one also needs to be able to sample from $q(\X_t | \X_0, \X_T)$.
\begin{proposition}[Bridge Distribution of the Time-Correlated Prior]\label{thm:correlated-bridge-distribution}
Let $\X_t$ follow the same SDE as in Proposition~\ref{thm:correlated-process-score}. Then, given fixed endpoints \( \X_0 \) and \( \X_{t'} \), $0 \le t < t'$, the posterior (bridge) distribution of \( \X_t \) is Gaussian:
\begin{equation}\label{eq:bridge-distr}
q(\X_t | \X_0, \X_{t'}) = \mathcal{N}(\X_t \mid \muvec_{t|0, t'}, \Si_{t|0, t'}),
\end{equation}
where, with the cross-covariance $\Si_{t,t'|0} = \mathrm{Cov}(\X_t, \X_{t'} \mid \X_0) = e^{\A(t'-t)}\Si_{t|0}$,
\begin{align}
\muvec_{t|0, t'} &= \muvec_{t|0}(\X_0) + \Si_{t,t'|0} \Si_{t'|0}^{-1} (\X_{t'} - \muvec_{t'|0}(\X_0)), \label{eq:bridge-mean}\\
\Si_{t|0, t'} &= \Si_{t|0} - \Si_{t,t'|0} \Si_{t'|0}^{-1} \Si_{t,t'|0}^\top. \label{eq:bridge-cov}
\end{align}
\end{proposition}

All matrices in Propositions~\ref{thm:correlated-process-score} and~\ref{thm:correlated-bridge-distribution} are functions of $\A$ and therefore commute; in particular, $\Si_{t|0,t'} = \Si_{t|0} - e^{2\A(t'-t)}\Si_{t|0}^2\Si_{t'|0}^{-1}$. These closed-form expressions are all that is required to implement bridge matching under the time-correlated prior, and they are used in all our experiments.

\subsection{Time-Correlated Video Bridge Matching}\label{sec:tcvbm_theory}

\paragraph{Training.} As in Section~\ref{sec:background}, we use a coupling \( p(\X_0, \X_T) = p_0(\X_0)p(\X_T|\X_0) \) of clean and degraded videos. For frame interpolation and image-to-video generation, $\bv$ is built from the known frames of each clip (Section~\ref{sec:theory_tasks}), which are also given to the network as conditioning; we omit this dependence from the notation.

We fit \( v_\phi(\X_t, t) \) to the score of the prior process over bridge samples \( \X_t \sim q(\X_t|\X_0, \X_T) \):
\begin{equation}
    \label{eq:initial-obj}
    \min_{\phi} \mathbb{E}_{\X_0, \X_t, t} \left[ \left\| v_\phi(\X_t, t) + \Si_{t|0}^{-1} (\X_t - \muvec_{t|0}(\X_0)) \right\|^2 \right],
\end{equation}
where \( t \sim \text{Uniform}(0, T) \) and $\|\cdot\|$ is the Frobenius norm.

This objective can be simplified by reparameterizing the score function in terms of a clean-data predictor:

\begin{proposition}[Reparameterization of the score function]
\label{thm:reparam}
Over all measurable functions, the minimizer \( v^*(\X_t, t) \) of the objective \eqref{eq:initial-obj} can be expressed as:
\[
v^*(\X_t, t) = -\Si_{t|0}^{-1} \left(\X_t - \muvec_{t|0}(\widehat{\X}_0^*(\X_t, t))\right),
\]
where \( \widehat{\X}_0^*(\X_t, t) = \mathbb{E}[\X_0 \mid \X_t] \) is the solution to the regression problem:
\begin{equation}
    \label{eq:final-obj}
    \min_{\phi} \mathbb{E}_{\X_0, \X_t, t}  \left[\|\widehat{\X}_0^{\phi}(\X_t, t) - \X_0\|^2 \right].
\end{equation}
Thus, learning the score function reduces to learning a predictor for the clean data \( \X_0 \).
\end{proposition}
As in I2SB and DDBM~\citep{i2sb, zhoudenoising}, this holds because the score of $q(\X_t|\X_0)$ is affine in $\X_0$.
We parameterize the predictor \( \widehat{\X}_0^{\phi}(\X_t, t) \) with a neural network and train it using the regression loss in \eqref{eq:final-obj}. The training procedure is summarized in Algorithm~\ref{alg:training}.

\paragraph{Inference.} Given a corrupted sequence \( \X_T \sim p(\X_T) \) and a time grid \( 0 = t_0 < t_1 < \cdots < t_K = T \), we iteratively refine the estimate of \( \X_0 \) by sampling from the bridge posterior
\(
\X_{t_{k-1}} \sim q(\X_{t_{k-1}} \mid \widehat{\X}_0, \X_{t_k}),
\)
where \( \widehat{\X}_0 = \widehat{\X}_0^{\phi}(\X_{t_k}, t_k) \) is the prediction of the trained model (Algorithm~\ref{alg:inference}). As in I2SB~\citep{i2sb}, this is an ancestral sampler that plugs the predicted $\widehat{\X}_0$ into the bridge of Proposition~\ref{thm:correlated-bridge-distribution}.

\begin{algorithm}[t]
  \caption{Training}
  \label{alg:training}
  \begin{algorithmic}[1]
    \REQUIRE data from coupling $p_0(\X_0)p(\X_T|\X_0)$ and coefficients $\A$, $\epsilon$, $\bv$ (per clip, Sec.~\ref{sec:theory_tasks}) for prior \eqref{eq:tc-prior}.
    \REPEAT
      \STATE $t \sim \mathcal{U}([0,T]), \X_0 \sim p_0(\X_0), \X_T \sim p(\X_T \mid \X_0)$
      \STATE $\X_t \sim q\bigl(\X_t \mid \X_0, \X_T\bigr)$ \eqref{eq:bridge-distr}
      \STATE Take gradient descent step on $\widehat{\X}_0^{\phi}(\X_t, t)$ \eqref{eq:final-obj}
    \UNTIL{convergence}
  \end{algorithmic}
\end{algorithm}

\begin{algorithm}[t]
  \caption{Inference}
  \label{alg:inference}
  \begin{algorithmic}[1]
    \REQUIRE Input \( \X_T \sim p_T(\X_T) \), known frames (through $\bv$), trained model \( \widehat{\X}_0^{\phi}(\cdot, \cdot) \), time grid \( \{t_k\}_{k=0}^K \)
    \STATE Set \( \X_{t_K} \leftarrow \X_T \)
    \FOR{$k = K$ \textbf{to} $1$}
      \STATE Predict \( \widehat{\X}_0 \leftarrow \widehat{\X}_0^{\phi}(\X_{t_k}, t_k) \)
      \STATE Sample \( \X_{t_{k-1}} \sim q(\X_{t_{k-1}} \mid \widehat{\X}_0, \X_{t_k}) \) using \eqref{eq:bridge-distr}
    \ENDFOR
    \STATE \textbf{return} \( \X_0 = \X_{t_0} \)
    \end{algorithmic}
\end{algorithm}

\subsection{Choice of the Prior Process for Video Manipulation Tasks}\label{sec:theory_tasks}

To encourage smooth transitions between consecutive frames, we define the $N \times N$ prior matrix as \( \alpha\A \), where $\alpha \ge 0$ controls the strength of the coupling and \( \A \) has a tridiagonal structure (a discrete Laplacian along the time axis):
\[
\A = \begin{bmatrix}
-2 & 1 & 0 & \cdots & 0 \\
1 & -2 & 1 & \ddots & \vdots \\
0 & \ddots & \ddots & \ddots & 0 \\
\vdots & \ddots & 1 & -2 & 1 \\
0 & \cdots & 0 & 1 & -2
\end{bmatrix}.
\]
For $\alpha=1$, the per-element prior is
\begin{equation}
\label{eq:video-prior-per-element}
d\x_t^n = \big( (\x_t^{n-1} - \x_t^n) + (\x_t^{n+1} - \x_t^n) \big)\, dt + \sqrt{\epsilon}\, d\W_t^n,
\end{equation}
for \( n = 1, \dots, N\), where the boundary values $\x_t^0$ and $\x_t^{N+1}$ are specified below for each task; they enter the drift through $\bv$. Each frame \( \x_t^n \) is thus pulled towards the average of its neighbors, which promotes consistency and smoothness between the frames of a video. For general $\alpha$, only $\A$ is multiplied by $\alpha$, while $\bv$ is kept fixed, so the boundary interpretation below is exact for $\alpha=1$, used in our frame interpolation and image-to-video experiments. As $\alpha\to0$, the prior tends to a Brownian motion with constant drift $\bv$, whose bridge does not depend on $\bv$, and the bridge of Proposition~\ref{thm:correlated-bridge-distribution} tends to the frame-independent Brownian bridge (Appendix~\ref{appendix:spectral}). Standard Bridge Matching is therefore recovered in the limit $\alpha\to0$, which we denote by $\alpha=0$.

\paragraph{Frame Interpolation.} We consider a sequence of length \( N + 2\), \(\widetilde{\X} = (\x^0, \x^1, \ldots, \x^N, \x^{N+1})\), where the endpoints \( \x^0 \) and \( \x^{N+1} \) are the fixed initial and final frames of a video clip, and the unknown middle part \( \X = (\x^1, \ldots, \x^N) \) is defined as in~\eqref{eq:interframes}, so all statements of Sections~\ref{sec:prior_theory} and~\ref{sec:tcvbm_theory} apply to $\X_t$. Setting \( \x_t^0 = \x^0 \) and \( \x_t^{N+1} = \x^{N+1} \) for all \( t \) in~\eqref{eq:video-prior-per-element} gives
\(\bv = [\x^0, 0, \dots, 0, \x^{N+1}]^\top \in \mathbb{R}^{N \times D},\)
i.e., for $\alpha=1$, \( \bv \) enforces the boundary conditions from the known endpoints.

\paragraph{Image-to-Video Generation.} Only the first frame $\x^0$ is fixed: \(\widetilde{\X} = (\x^0, \x^1, \ldots, \x^N)\), \(\bv = [\x^0, 0, \dots, 0]^\top \in \mathbb{R}^{N \times D}\), i.e., $\x_t^0 = \x^0$ and $\x_t^{N+1} = 0$ in~\eqref{eq:video-prior-per-element}.

\paragraph{Video Super-Resolution.} No frame is fixed: \( \widetilde{\X} = \X \in \mathbb{R}^{N \times D} \) and $\bv = 0$, which corresponds to zero boundary values $\x_t^0 = \x_t^{N+1} = 0$. For zero-mean normalized data, the zero boundary values used for image-to-video generation and super-resolution act as a weak pull towards the mean frame.

\subsection{Spectral View of the Prior}\label{sec:spectral}

The structure of $\A$ makes the effect of the prior explicit. $\A$ is diagonalized by the discrete sine transform~\citep{strang2007computational}: $\A = \mathbf{U}\boldsymbol{\Lambda}\mathbf{U}^\top$ with $\mathbf{U}_{nk} = \sqrt{2/(N+1)}\,\sin\big(\pi nk/(N+1)\big)$ and eigenvalues $\lambda_k = -4\sin^2\big(\pi k/(2(N+1))\big) \in (-4, 0)$, $k=1,\dots,N$; index $k$ is a temporal frequency. Since all quantities in Propositions~\ref{thm:correlated-process-score}--\ref{thm:correlated-bridge-distribution} are functions of $\A$, in the coordinates $\mathbf{z} = \mathbf{U}^\top\x^{(d)}$ the prior decouples into $N$ independent scalar Ornstein--Uhlenbeck processes $dz_k = (\alpha\lambda_k z_k + \beta_k)\,dt + \sqrt{\epsilon}\,dw_k$ with $\beta_k = \mathbf{u}_k^\top\bv^{(d)}$, where $\mathbf{u}_k$ is the $k$-th column of $\mathbf{U}$ and $w_k$ are independent Wiener processes (Appendix~\ref{appendix:spectral}). This yields three properties.

\textbf{(i) Temporal low-pass mean.} With $\mathbf{m}^* = -(\alpha\A)^{-1}\bv$, the mean~\eqref{eq:main-mean} reads $\muvec_{t|0} = e^{\alpha\A t}\X_0 + (\I - e^{\alpha\A t})\,\mathbf{m}^*$: the $k$-th temporal frequency of $\X_0$ relaxes towards that of $\mathbf{m}^*$ at rate $\alpha|\lambda_k|$, so high temporal frequencies are damped first. The prior thus acts as a heat equation along the time axis, the temporal counterpart of blurring and heat-dissipation diffusion~\citep{rissanen2023ihdm, hoogeboom2023blurring}, here used as the reference process of a bridge between data distributions.

\textbf{(ii) Temporally smooth noise.} The variance of frequency $k$ is $\epsilon(1-e^{-2\alpha|\lambda_k| t})/(2\alpha|\lambda_k|)$: close to the Brownian value $\epsilon t$ for low frequencies but bounded by $\epsilon/(2\alpha|\lambda_k|)$ for high ones. The stationary covariance $-\tfrac{\epsilon}{2\alpha}\A^{-1}$ equals, up to the factor $\tfrac{\epsilon}{2\alpha}$, the covariance $\min(i,j)\,(N{+}1{-}\max(i,j))/(N{+}1)$ of a discrete Brownian bridge across frames, whereas correlated-noise priors for video diffusion use a shared noise component or an autoregressive correlation between frames~\citep{ge2024preservecorrelationnoiseprior}; in TCVBM, the correlation is a consequence of the coupled drift and affects both the mean and the noise of the bridge. In statistical terms, the stationary prior over frames is a first-order random walk pinned at the boundary frames, i.e., a Gaussian Markov random field whose precision matrix is proportional to the discrete Laplacian $-\A$, a classical smoothing prior~\citep{rue2005gmrf}.

\textbf{(iii) Boundary-determined profile.} $\alpha\mathbf{m}^*$ solves the discrete Laplace equation with the boundary values of Section~\ref{sec:theory_tasks}. For $\alpha=1$, used in our frame interpolation and image-to-video experiments, $\mathbf{m}^*$ is, for frame interpolation, exactly the linear interpolation $\x^n = \x^0 + \tfrac{n}{N+1}(\x^{N+1}-\x^0)$ between the key frames; for image-to-video generation, it is a linear decay of $\x^0$ towards zero; for super-resolution, it is zero. For general $\alpha$, the profile is scaled by $1/\alpha$.

Finally, $\alpha \to 0$ gives $\alpha\lambda_k \to 0$ for every $k$, recovering frame-independent BM; larger $\alpha$ gives stronger temporal smoothing. In the eigenbasis, the bridge~\eqref{eq:bridge-distr} is computed independently per frequency, at the cost of two $N \times N$ matrix products per step.

\section{Experiments}

We consider three video-to-video tasks with the boundary conditions of Section~\ref{sec:theory_tasks}. Let $L$ denote the clip length. In \textbf{frame interpolation}, the first and last frames are given and the $L-2$ intermediate frames are generated and compared with the ground truth; in \textbf{image-to-video generation}, the remaining $L-1$ frames are generated from the first one; in \textbf{video super-resolution}, $L$ low-resolution frames are mapped to $L$ frames at the dataset resolution. Throughout, BM denotes Bridge Matching with the frame-independent Brownian bridge prior, i.e., the limit $\alpha\to0$ of TCVBM, trained and sampled in exactly the same way as TCVBM; all bridge-based methods use the same number of sampling steps and $T=1$.

\subsection{Proof-of-Concept Experiments on Moving MNIST}\label{sec:simple_experiments}

First, we compare TCVBM with DDPM~\citep{10.5555/3495724.3496298}, DDIM~\citep{song2021denoising}, and BM on the Moving MNIST dataset~\citep{DBLP:journals/corr/SrivastavaMS15}, which contains $10{,}000$ sequences of $20$ frames with $2$ moving digits at $64 \times 64$ resolution. All methods share the same small 2D U-Net (8.7M parameters, $L=10$ frames stacked as channels), data and optimization, and differ only in the loss and the sampling procedure. The network has no temporal layers; it can still mix the frames stacked as channels, but since all methods share it, their differences isolate the effect of the prior. Implementation details are given in Appendix~\ref{sec:mnist_implementation}.

For frame interpolation, the intermediate frames of $\X_T$ are Gaussian noise $\mathcal{N}(\mathbf{0}, \I)$, which worked better for both bridge-based methods than a linear interpolation of the boundary frames (Appendix~\ref{appendix:initialization_interpolation}). For image-to-video generation, $\X_T$ consists of copies of the first frame, which worked better than Gaussian noise for both bridge-based methods (Appendix~\ref{appendix:initialization_i2v}). For video super-resolution, we report results for $16\times16$ inputs upsampled to $64\times64$; at $32\times32$ the task is too easy to separate the methods. Feeding only the low-resolution video or concatenating it with Gaussian noise gave similar results, so we use the former (Appendix~\ref{appendix:initialization_sr}).

We use $\epsilon=1$ and $\alpha=1$ for frame interpolation and image-to-video generation, and $\epsilon=0.1$ for both bridge-based methods with $\alpha=0.1$ for TCVBM in super-resolution. A study of $\epsilon$ and $\alpha$ is given in Appendix~\ref{appendix:hyperparameters}; for the large-scale experiments, we did not tune these hyperparameters. A time-dependent coupling $f(t)$ did not improve over a constant $\alpha$ (Appendix~\ref{appendix:dynamical_correlation}); the computational overhead is analyzed in Appendix~\ref{appendix:comp_cost}.

\subsection{Large-Scale Experiments}\label{sec:large_scale_experiments}

Second, we combine TCVBM with video architectures on real data. In both cases, the prior acts on latent frames: $\A$ is defined over the latent frames of the unknown part of the clip, and the latents of the known frames enter $\bv$.

\paragraph{Frame Interpolation.} We use UCF-101~\citep{soomro2012ucf101dataset101human} (9,537 training and 3,783 test clips, resized to $256 \times 256$) and the Vimeo-90K septuplet dataset~\citep{Xue_2019} (91,701 7-frame sequences at $448 \times 256$). We train on UCF-101 only ($L = 10$) and on the union of UCF-101 septuplets and the Vimeo-90K training set ($L = 7$). We fine-tune the pretrained CogVideoX-2B~\citep{yang2025cogvideoxtexttovideodiffusionmodels} for 10k iterations with batch size 4 and gradient accumulation 4 on one NVIDIA A100 GPU, with the same AdamW configuration as on Moving MNIST, $\epsilon=0.1$, $\alpha=1$ and $K=100$ sampling steps. The first and last frames are encoded separately by the CogVideoX 3D VAE and enter $\bv$, while the intermediate frames are encoded jointly; since the VAE compresses time by a factor of 4, $\A$ acts on a small number of latent frames (details in Appendix~\ref{sec:mnist_implementation}). Existing bridge models for frame interpolation~\citep{lyu2024cbbd, lyu2025tlbvfi} also use a frame-independent Brownian reference but start from neighboring-frame latents rather than from noise; our BM baseline shares their reference process under our architecture, which isolates the effect of the prior. As specialized interpolation models, we use RIFE~\citep{huang2022rife} (flow-based) and LDMVFI~\citep{10.1609/aaai.v38i2.27912} (latent diffusion), retrained on the same data with their official implementations; unlike them, TCVBM starts from a pretrained video model.

\paragraph{Image-to-Video Generation.} We train and test on the train-test split of UCF-101 and compare with FrameBridge~\citep{wang2025framebridge}, a bridge model for image-to-video generation with a frame-independent prior. Both methods use the Latte-S/2 architecture~\citep{ma2025latte} on frame-wise latents and are trained from scratch with Xavier initialization for 400k iterations with batch size 20; FrameBridge is trained with its official code, and TCVBM replaces its frame-independent bridge with the time-correlated one (Appendix~\ref{sec:mnist_implementation}).

\begin{figure*}
\begin{subfigure}{.325\textwidth}
  \centering
  \includegraphics[width=\linewidth]{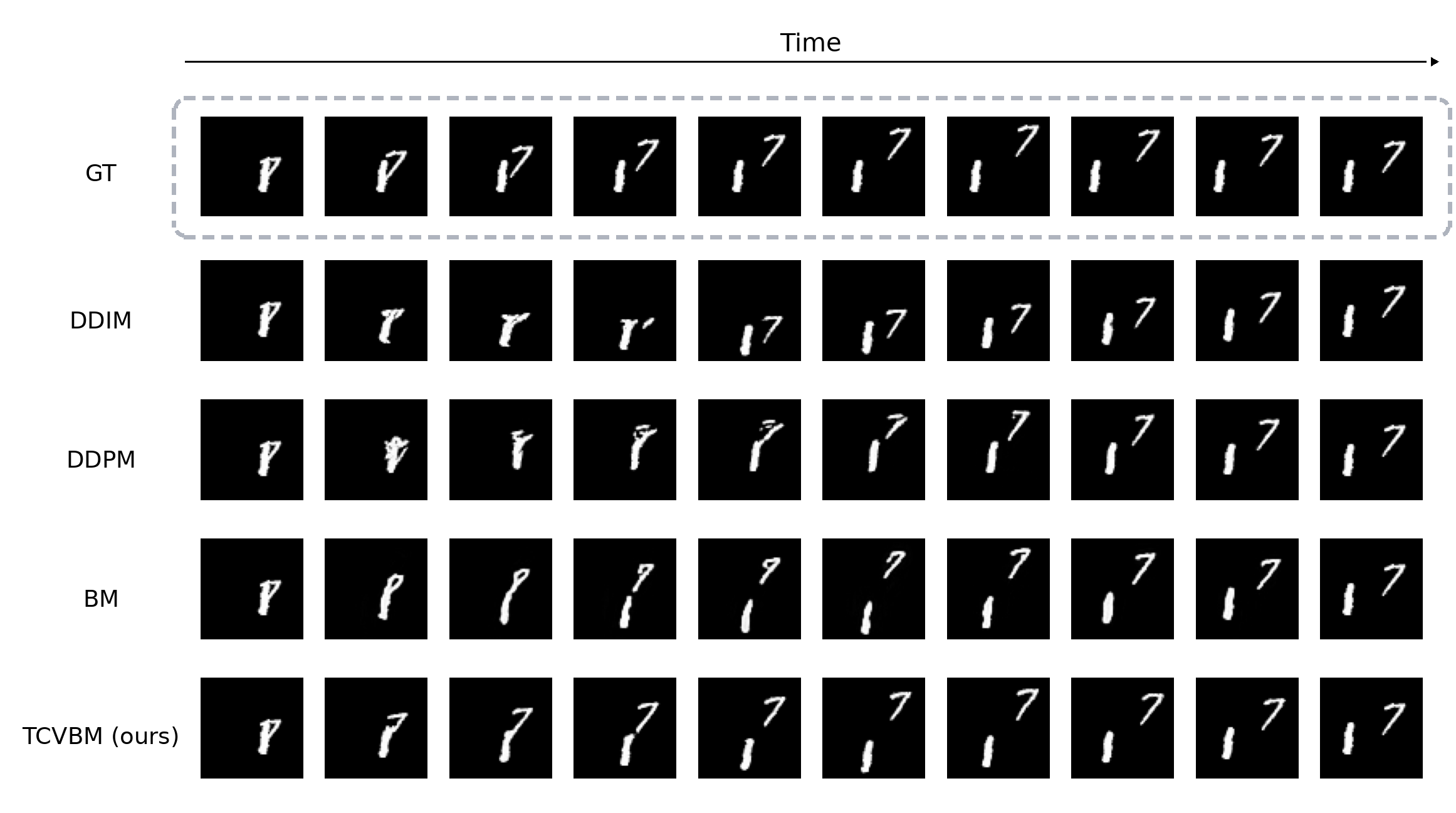}
  \caption{Frame interpolation: the first and last frames are given, and the $8$ intermediate frames are generated. TCVBM gives smoother transitions and finer details.}
  \label{fig:interpolation_results}
\end{subfigure}\hfill
\begin{subfigure}{.325\textwidth}
  \centering
  \includegraphics[width=\linewidth]{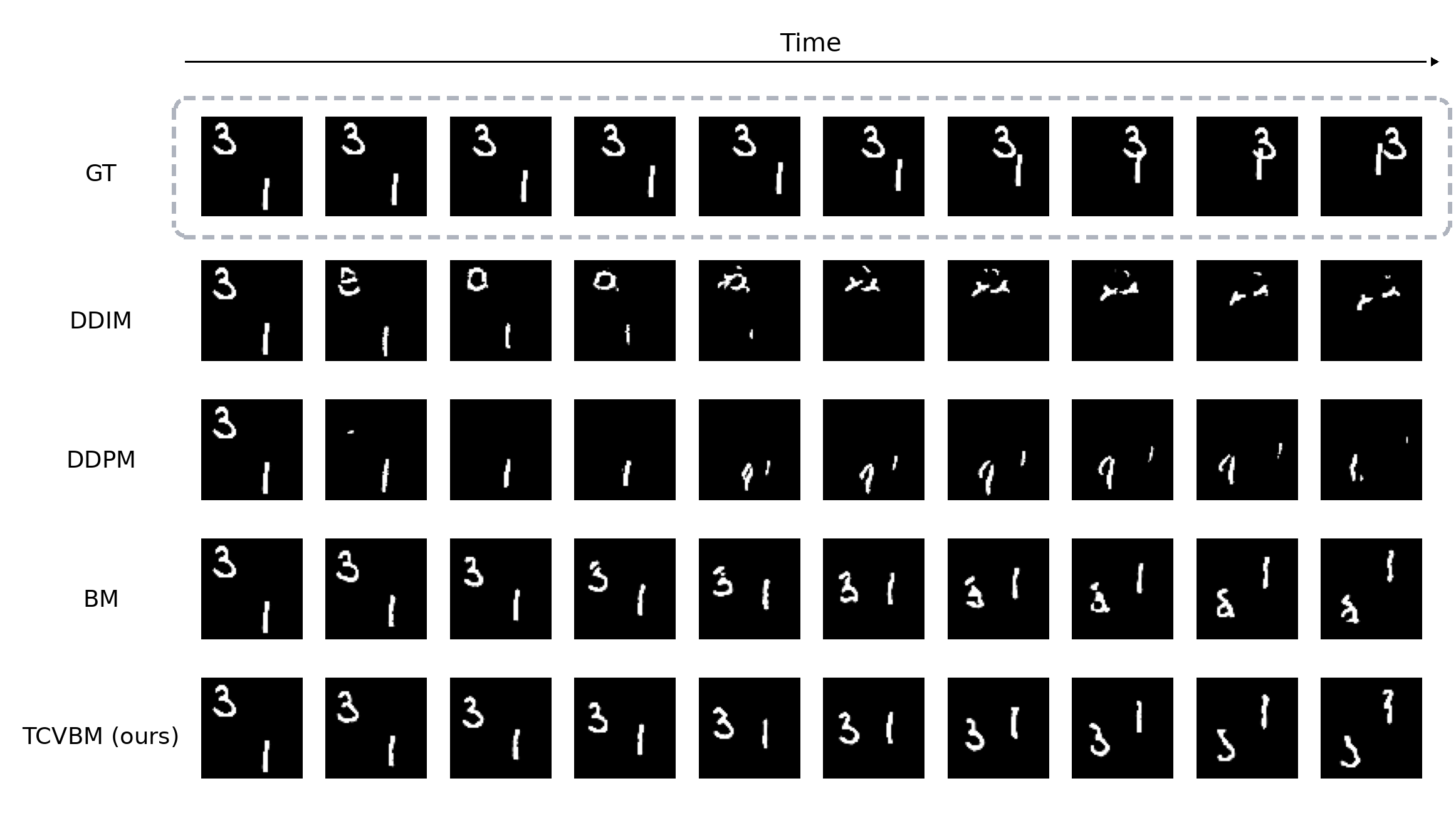}
  \caption{Image-to-video generation: $9$ frames are generated from the first one. TCVBM yields sharper frames and degrades less along the sequence.}
  \label{fig:i2v_results}
\end{subfigure}\hfill
\begin{subfigure}{.31\textwidth}
  \centering
  \includegraphics[width=\linewidth]{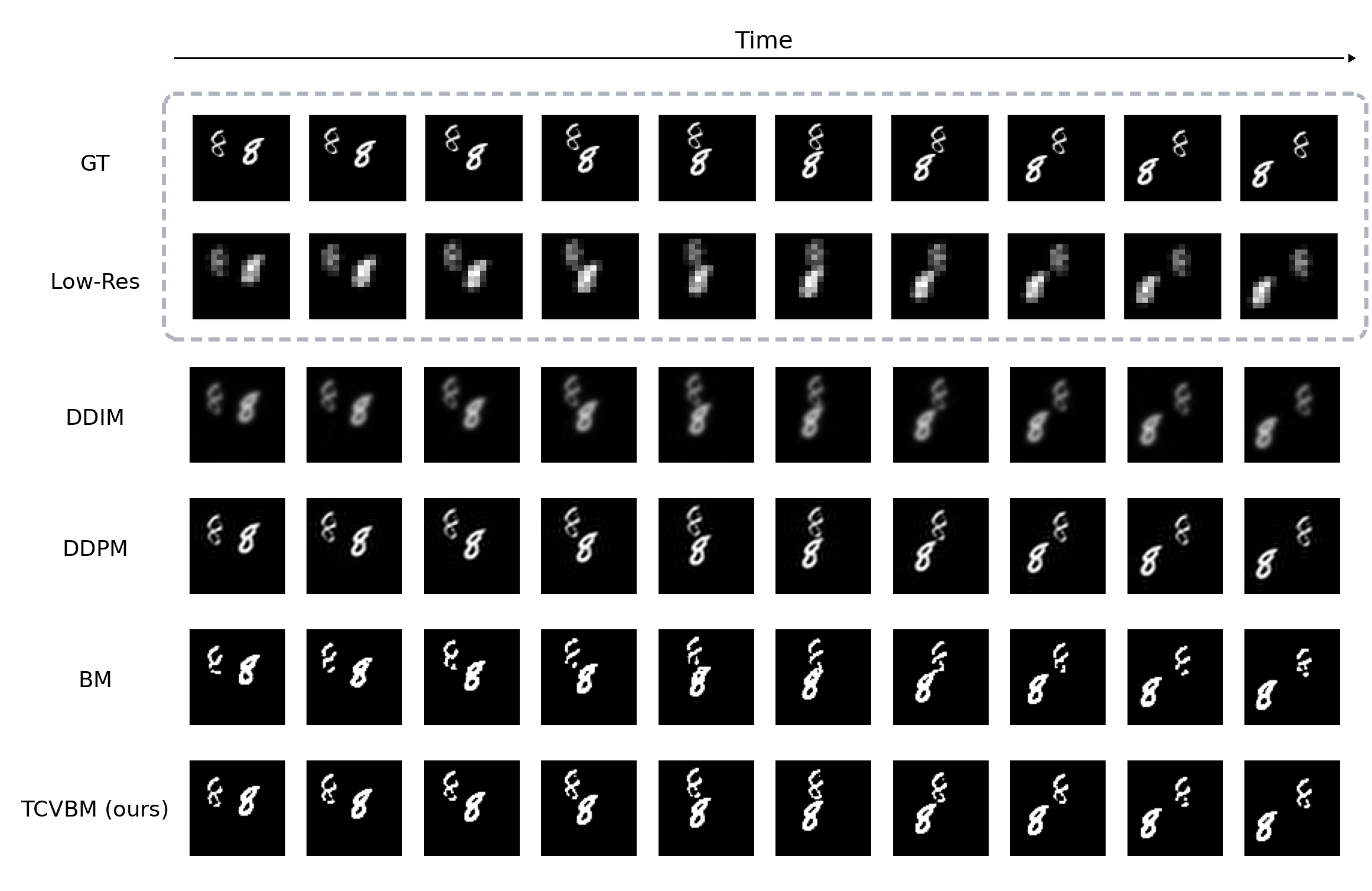}
  \caption{Video super-resolution from $16 \times 16$ to $64 \times 64$: TCVBM and BM perform on par and outperform DDPM and DDIM.}
  \label{fig:sr_results}
\end{subfigure}
\caption{Generation results on the Moving MNIST dataset.}
\label{fig:results}
\end{figure*}

\begin{table*}[h]
  \centering
    \caption{Results on the Moving MNIST dataset (mean over 5 seeds; standard deviations in Appendix~\ref{appendix:confidence}).}\label{tab:moving_mnist_results}
  \footnotesize
  \setlength{\tabcolsep}{3pt}
    \begin{tabular}{l cccc cccc cccc}
    \hline
    Method &
    \multicolumn{4}{c}{Frame interpolation} &
    \multicolumn{4}{c}{Image-to-video generation} &
    \multicolumn{4}{c}{Video super-resolution} \\
    \cmidrule(lr){2-5}\cmidrule(lr){6-9}\cmidrule(lr){10-13} &
    FVD$\downarrow$ &
    LPIPS$\downarrow$ &
    PSNR$\uparrow$ &
    SSIM$\uparrow$ &
    FVD$\downarrow$ &
    LPIPS$\downarrow$ &
    PSNR$\uparrow$ &
    SSIM$\uparrow$ &
    FVD$\downarrow$ &
    LPIPS$\downarrow$ &
    PSNR$\uparrow$ &
    SSIM$\uparrow$ \\
    \hline
    DDIM  & 33.61   &  0.105  & 15.80  & 0.760   & 77.72  & 0.294  & \textbf{10.88} & \textbf{0.603} & 334.70 & 0.514 & 17.31 & 0.613 \\
    DDPM    & 32.40 & 0.117 & 14.94  &  0.741 & 75.86 & 0.311 & 10.72  & 0.595 & 607.88 & 0.236 & 20.19 & 0.582\\
    BM  & 34.32 &  0.079 & 17.10  &  0.794 & 49.32  & 0.271  &  10.63  &  0.579 & \textbf{53.52} & \textbf{0.020} & 22.64 & 0.954 \\
    \hline
    TCVBM (ours) & \textbf{30.54} & \textbf{0.077} & \textbf{17.28} & \textbf{0.813} & \textbf{44.96 }  & \textbf{0.258} & 10.75 & 0.591 & 59.49 & \textbf{0.020} & \textbf{22.67} & \textbf{0.970}\\
    \hline
    \end{tabular}
\end{table*}
\section{Results}\label{sec:results}

\paragraph{Qualitative Comparison.} TCVBM is temporally more consistent than the frame-independent baselines in frame interpolation and image-to-video generation and on par in super-resolution (Figures~\ref{fig:results} and~\ref{fig:ucf_results}).

\paragraph{Quantitative Evaluation.} First, we evaluate the proof-of-concept experiments on $500$ videos from our Moving MNIST validation set using FVD~\citep{unterthiner2019accurategenerativemodelsvideo}, LPIPS~\citep{zhang2018unreasonable}, PSNR, and SSIM (Table~\ref{tab:moving_mnist_results}; mean over 5 seeds, std in Appendix~\ref{appendix:confidence}). For image-to-video generation, TCVBM improves over BM on FVD and LPIPS by more than the seed-to-seed standard deviation. For frame interpolation, TCVBM has the best mean values on all four metrics, but its differences from BM are within one standard deviation. For video super-resolution, TCVBM and BM perform on par on automatic metrics, while human raters prefer TCVBM (Appendix~\ref{sec:mnist_human}).

\begin{table}[h]
  \centering
    \caption{Frame interpolation results on the UCF-101 test set for the model trained on the UCF-101 train set.}\label{tab:interpolation_results_ucf}
    \small\setlength{\tabcolsep}{4pt}
    \begin{tabular}{lcccc}
    \hline
    Method &
    {FVD$\downarrow$} &
    {LPIPS$\downarrow$} &
    {PSNR$\uparrow$} &
    {SSIM$\uparrow$} \\
    \hline
    BM  & 688.26 &  \textbf{0.0583} & 34.51 & 0.9353 \\
    \hline
    TCVBM (ours) & \textbf{611.48} & 0.0588 & \textbf{35.72} & \textbf{0.9574} \\
    \hline    \end{tabular}
\end{table}

\begin{table*}[h]
  \centering
    \caption{Frame interpolation results on the UCF-101 test set and the Vimeo-90K septuplet test set for the models trained on the union of the UCF-101 and Vimeo-90K training sets. All baselines were retrained on the same data with their official implementations.}\label{tab:vimeo_results_intepolation}
  \setlength{\tabcolsep}{4pt}
    \begin{tabular}{l ccccc cccc}
    \hline
    Method &
    \multicolumn{5}{c}{UCF-101} &
    \multicolumn{3}{c}{Vimeo-90K} \\
    \cmidrule(lr){2-6}\cmidrule(lr){7-9} &
    PSNR$\uparrow$ &
    SSIM$\uparrow$ &
    LPIPS$\downarrow$ &
    FloLPIPS$\downarrow$ &
    FVD$\downarrow$ &
    PSNR$\uparrow$ &
    SSIM$\uparrow$ &
    LPIPS$\downarrow$ \\
    \hline
    RIFE~\citep{huang2022rife}  & 35.06 & 0.9489 & 0.044 & 0.041 & 552 & 35.83 & 0.904 & 0.098\\
    LDMVFI~\citep{10.1609/aaai.v38i2.27912} & 33.47 & 0.9509 & \textbf{0.036} & \textbf{0.035} & \textbf{369} & 35.26 & 0.909 & \textbf{0.094} \\
    BM  & 30.59 & 0.9134 & 0.052 & 0.059 & 678 & 31.51 & 0.829 & 0.113 \\
    \hline
    TCVBM (ours)  & \textbf{35.63} & \textbf{0.9513} & 0.041 & 0.037 & 441 & \textbf{35.84} & \textbf{0.911} & 0.097\\
    \hline
    \end{tabular}
\end{table*}

\begin{table*}[h]
\centering
    \caption{Image-to-video generation on UCF-101 with the metrics of~\citet{wang2025framebridge}. FrameBridge was retrained with the official code under the same Latte-S/2 backbone, data and budget as TCVBM (batch size 20, $256\times256$); the values reported by~\citet{wang2025framebridge} for Latte-S/2 under their original configuration (FVD 154 for FrameBridge, 122 with a neural prior, 171 for VDT-I2V) are not comparable to this budget. Rows marked with $^\dagger$ are reported by~\citet{wang2025framebridge} (I2VGen-XL and SparseCtrl zero-shot) for reference.}\label{tab:i2v_results_ucf}
    \begin{tabular}{lcccc}
    \hline
    Method &
    {FVD$\downarrow$} &
    {IS$\uparrow$} &
    {PIC$\uparrow$} \\
    \hline
    I2VGen-XL$^\dagger$~\citep{2023i2vgenxl} & 571 & -- & 0.5313 \\
    SparseCtrl$^\dagger$~\citep{guo2023sparsectrladdingsparsecontrols} & 722 & 19 & 0.4818 \\
    \hline
    FrameBridge (retrained)~\citep{wang2025framebridge}  & 603 & 36 & 0.6857 \\
    TCVBM (ours) & \textbf{467} & \textbf{59} & \textbf{0.7441} \\
    \hline \end{tabular}
\end{table*}

\begin{table*}[h!]
  \centering
    \caption{Human evaluation on 40 videos from the UCF-101 test set: share of votes (\%) preferring each method or both (10 assessors, 400 votes per column). Baseline: BM (models of Table~\ref{tab:interpolation_results_ucf}) for frame interpolation and FrameBridge for image-to-video generation. $p$: two-sided sign test of TCVBM vs.\ baseline votes, ties excluded.}\label{tab:human_ucf}
  \setlength{\tabcolsep}{5pt}
    \begin{tabular}{l cc cc}
    \hline
    Method &
    \multicolumn{2}{c}{Frame interpolation} &
    \multicolumn{2}{c}{Image-to-video generation} \\
    \cmidrule(lr){2-3}\cmidrule(lr){4-5} &
    Frame quality &
    Temporal consistency &
    Frame quality &
    Temporal consistency \\
    \hline
    Baseline  &  19\% & 15\% & 36\% & 30\% \\
    TCVBM (ours) & 41\% & 49\% & 37.5\% & 33\% \\
    Both & 40\% & 36\% & 26.5\% & 37\% \\
    \hline
    $p$ & $<10^{-7}$ & $<10^{-17}$ & 0.77 & 0.49 \\
    \hline
    \end{tabular}
\end{table*}

On Moving MNIST, the gains beyond one standard deviation are in image-to-video generation, whereas on UCF-101 humans prefer TCVBM only for frame interpolation; the spectral view (Section~\ref{sec:spectral}) suggests, as an untested hypothesis, that the prior is most informative when the known frames fix a non-trivial profile.

\begin{figure}[t]
  \centering
  \includegraphics[width=\linewidth]{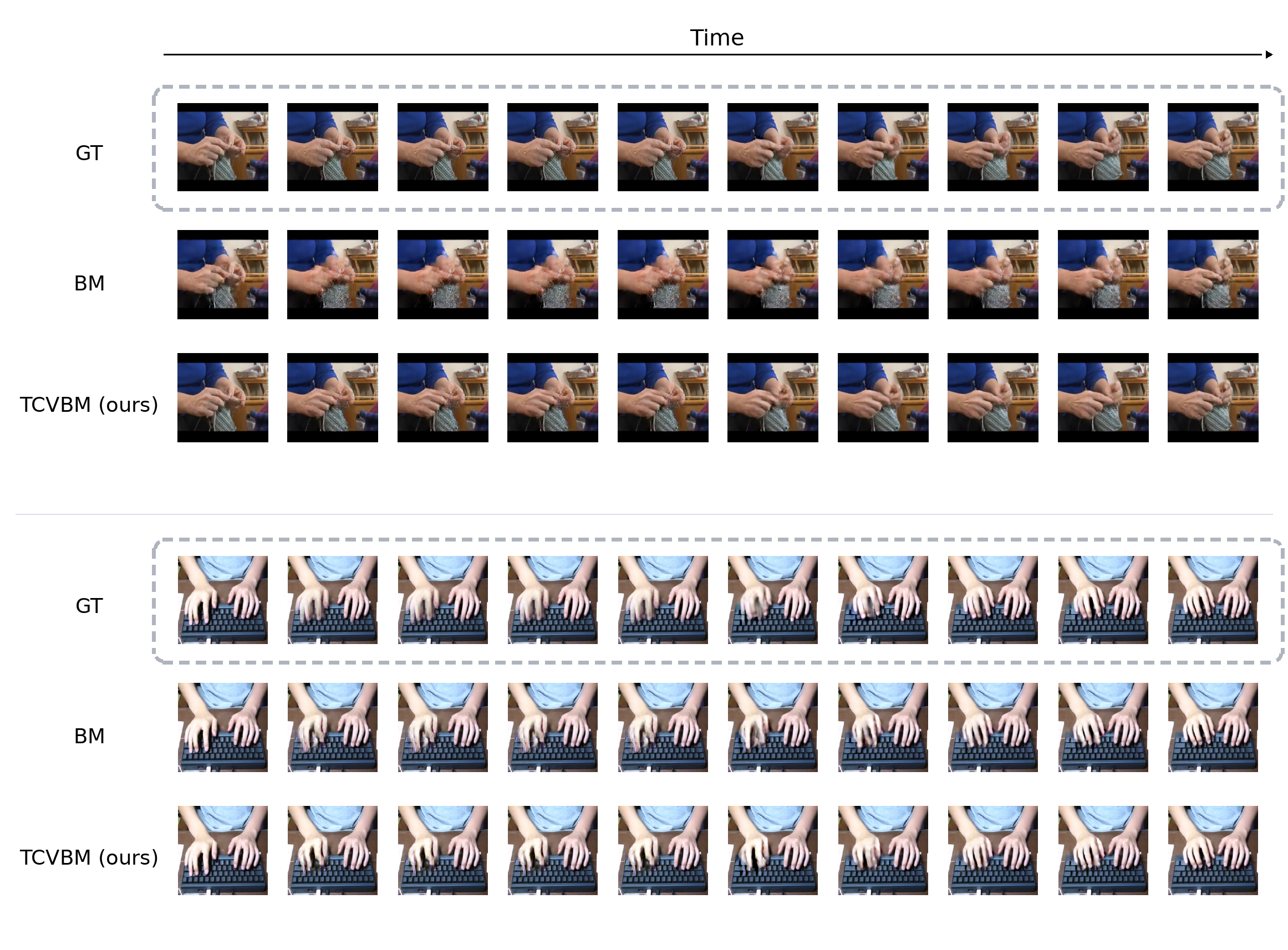}
  \caption{Frame interpolation results on the UCF-101 test dataset, trained on the UCF-101 training set. TCVBM retains more structural information by linking to boundary frames, providing improved detail and dynamics.}
  \label{fig:ucf_results}
\end{figure}

Tables~\ref{tab:interpolation_results_ucf}--\ref{tab:i2v_results_ucf} report the large-scale experiments. For frame interpolation, we additionally use FloLPIPS, which accounts for motion consistency via optical flow~\citep{flolpips}; for image-to-video generation, we use FVD, Inception Score~\citep{saito2017temporalgenerativeadversarialnets}, and PIC~\citep{xing2023dynamicrafter}, as in~\citet{wang2025framebridge}. In the controlled comparisons with frame-independent bridges, TCVBM improves over BM on all frame interpolation metrics except LPIPS in Table~\ref{tab:interpolation_results_ucf} (where the two are within 0.001), and over the retrained FrameBridge on all image-to-video metrics. Against specialized interpolation models, TCVBM attains the highest PSNR and SSIM, with near ties on Vimeo-90K, while LDMVFI is better on LPIPS, FloLPIPS and FVD. Repeated runs with identical settings preserved the ranking of BM and TCVBM on every metric.

\paragraph{Human Evaluation.} In a blind side-by-side study, ten assessors compared TCVBM with the bridge baseline of each task (Table~\ref{tab:human_ucf}) on 40 videos per task generated from UCF-101 test frames, with randomized left/right order, rating every pair for frame quality and temporal consistency (one video or both). TCVBM is strongly preferred for frame interpolation but not significantly for image-to-video generation (Table~\ref{tab:human_ucf}; the sign test treats votes as independent and is approximate). Results on Moving MNIST are in Appendix~\ref{sec:mnist_human}.

\section{Discussion}

\paragraph{Summary.} TCVBM places a discrete temporal Laplacian into the reference process of Bridge Matching, keeps the bridge in closed form, and contains BM as the limit $\alpha\to0$. It yields the clearest human preference for frame interpolation on UCF-101 and the clearest metric gains for image-to-video generation on Moving MNIST, but no significant human preference in image-to-video generation on UCF-101 and no metric gains in super-resolution.

\paragraph{Limitations.} First, the prior favors smooth videos with persistent content, which suits interpolation and enhancement but not scene changes, appearing or disappearing objects, or long-video generation. Second, the fixed tridiagonal prior captures only local linear dependencies; for latent video models with strong temporal compression, it couples only a few latent frames. Correlations in a learned feature space are a direction for future work. Third, $\epsilon$ and $\alpha$ were not tuned for the large-scale experiments, and a frame-independent Ornstein--Uhlenbeck control would further isolate the role of the coupling.

\bibliography{references}

@inproceedings{lipman2023flow,
    title={Flow Matching for Generative Modeling},
    author={Yaron Lipman and Ricky T. Q. Chen and Heli Ben-Hamu and Maximilian Nickel and Matthew Le},
    booktitle={The Eleventh International Conference on Learning Representations },
    year={2023},
    url={https://openreview.net/forum?id=PqvMRDCJT9t}
}

@inproceedings{10.5555/3495724.3496298,
author = {Ho, Jonathan and Jain, Ajay and Abbeel, Pieter},
title = {Denoising diffusion probabilistic models},
year = {2020},
isbn = {9781713829546},
publisher = {Curran Associates Inc.},
address = {Red Hook, NY, USA},
booktitle = {Proceedings of the 34th International Conference on Neural Information Processing Systems},
articleno = {574},
numpages = {12},
location = {Vancouver, BC, Canada},
series = {NIPS '20}
}

@inproceedings{
  song2021scorebased,
  title={Score-Based Generative Modeling through Stochastic Differential Equations},
  author={Yang Song and Jascha Sohl-Dickstein and Diederik P Kingma and Abhishek Kumar and Stefano Ermon and Ben Poole},
  booktitle={International Conference on Learning Representations},
  year={2021},
  url={https://openreview.net/forum?id=PxTIG12RRHS}
}

@misc{flux2024,
    author={{Black Forest Labs}},
    title={FLUX},
    year={2024},
    howpublished={\url{https://github.com/black-forest-labs/flux}},
}

@misc{saharia2022photorealistictexttoimagediffusionmodels,
      title={Photorealistic Text-to-Image Diffusion Models with Deep Language Understanding}, 
      author={Chitwan Saharia and William Chan and Saurabh Saxena and Lala Li and Jay Whang and Emily Denton and Seyed Kamyar Seyed Ghasemipour and Burcu Karagol Ayan and S. Sara Mahdavi and Rapha Gontijo Lopes and Tim Salimans and Jonathan Ho and David J Fleet and Mohammad Norouzi},
      year={2022},
      eprint={2205.11487},
      archivePrefix={arXiv},
      primaryClass={cs.CV},
      url={https://arxiv.org/abs/2205.11487}, 
}

@misc{rombach2022highresolutionimagesynthesislatent,
      title={High-Resolution Image Synthesis with Latent Diffusion Models}, 
      author={Robin Rombach and Andreas Blattmann and Dominik Lorenz and Patrick Esser and Björn Ommer},
      year={2022},
      eprint={2112.10752},
      archivePrefix={arXiv},
      primaryClass={cs.CV},
      url={https://arxiv.org/abs/2112.10752}, 
}

@inproceedings{vladimir-etal-2024-kandinsky,
    title = "Kandinsky 3: Text-to-Image Synthesis for Multifunctional Generative Framework",
    author = "Arkhipkin, Vladimir  and
      Vasilev, Viacheslav  and
      Filatov, Andrei  and
      Pavlov, Igor  and
      Agafonova, Julia  and
      Gerasimenko, Nikolai  and
      Averchenkova, Anna  and
      Mironova, Evelina  and
      Bukashkin, Anton  and
      Kulikov, Konstantin  and
      Kuznetsov, Andrey  and
      Dimitrov, Denis",
    editor = "Hernandez Farias, Delia Irazu  and
      Hope, Tom  and
      Li, Manling",
    booktitle = "Proceedings of the 2024 Conference on Empirical Methods in Natural Language Processing: System Demonstrations",
    month = nov,
    year = "2024",
    address = "Miami, Florida, USA",
    publisher = "Association for Computational Linguistics",
    url = "https://aclanthology.org/2024.emnlp-demo.48/",
    doi = "10.18653/v1/2024.emnlp-demo.48",
    pages = "475--485",
}

@misc{yang2025cogvideoxtexttovideodiffusionmodels,
      title={CogVideoX: Text-to-Video Diffusion Models with An Expert Transformer}, 
      author={Zhuoyi Yang and Jiayan Teng and Wendi Zheng and Ming Ding and Shiyu Huang and Jiazheng Xu and Yuanming Yang and Wenyi Hong and Xiaohan Zhang and Guanyu Feng and Da Yin and Yuxuan Zhang and Weihan Wang and Yean Cheng and Bin Xu and Xiaotao Gu and Yuxiao Dong and Jie Tang},
      year={2025},
      eprint={2408.06072},
      archivePrefix={arXiv},
      primaryClass={cs.CV},
      url={https://arxiv.org/abs/2408.06072}, 
}

@inproceedings{ge2024preservecorrelationnoiseprior,
  title={Preserve Your Own Correlation: A Noise Prior for Video Diffusion Models},
  author={Ge, Songwei and Nah, Seungjun and Liu, Guilin and Poon, Tyler and Tao, Andrew and Catanzaro, Bryan and Jacobs, David and Huang, Jia-Bin and Liu, Ming-Yu and Balaji, Yogesh},
  booktitle={Proceedings of the IEEE/CVF International Conference on Computer Vision (ICCV)},
  pages={22930--22941},
  year={2023}
}

@ARTICLE{10815947,
  author={Arkhipkin, Vladimir and Shaheen, Zein and Vasilev, Viacheslav and Dakhova, Elizaveta and Sobolev, Konstantin and Kuznetsov, Andrey and Dimitrov, Denis},
  journal={IEEE Access}, 
  title={ImproveYourVideos: Architectural Improvements for Text-to-Video Generation Pipeline}, 
  year={2025},
  volume={13},
  number={},
  pages={1986-2003},
  doi={10.1109/ACCESS.2024.3522510}}

@article{peluchetti2023diffusion,
  title={Diffusion bridge mixture transports, Schr{\"o}dinger bridge problems and generative modeling},
  author={Peluchetti, Stefano},
  journal={Journal of Machine Learning Research},
  volume={24},
  number={374},
  pages={1--51},
  year={2023}
}

@inproceedings{zhang2018unreasonable,
  title={The unreasonable effectiveness of deep features as a perceptual metric},
  author={Zhang, Richard and Isola, Phillip and Efros, Alexei A and Shechtman, Eli and Wang, Oliver},
  booktitle={Proceedings of the IEEE conference on computer vision and pattern recognition},
  pages={586--595},
  year={2018}
}

@techreport{petersen2008matrix,
  title={The Matrix Cookbook},
  author={Petersen, Kaare Brandt and Pedersen, Michael Syskind},
  institution={Technical University of Denmark},
  year={2012},
  note={Version November 15, 2012}
}

@inproceedings{
    shi2023diffusion,
    title={Diffusion Schr\"odinger Bridge Matching},
    author={Yuyang Shi and Valentin De Bortoli and Andrew Campbell and Arnaud Doucet},
    booktitle={Thirty-seventh Conference on Neural Information Processing Systems},
    year={2023},
    url={https://openreview.net/forum?id=qy07OHsJT5}
}

@article{peluchetti2023non,
  title={Non-denoising forward-time diffusions},
  author={Peluchetti, Stefano},
  journal={arXiv preprint arXiv:2312.14589},
  year={2023}
}

@inproceedings{zhoudenoising,
  title={Denoising Diffusion Bridge Models},
  author={Zhou, Linqi and Lou, Aaron and Khanna, Samar and Ermon, Stefano},
  year={2024},
  booktitle={The Twelfth International Conference on Learning Representations}
}

@inproceedings{
liu2022let,
title={Let us Build Bridges:  Understanding and Extending Diffusion Generative Models},
author={Xingchao Liu and Lemeng Wu and Mao Ye and qiang liu},
booktitle={NeurIPS 2022 Workshop on Score-Based Methods},
year={2022},
url={https://openreview.net/forum?id=0ef0CRKC9uZ}
}

@inproceedings{
song2021denoising,
title={Denoising Diffusion Implicit Models},
author={Jiaming Song and Chenlin Meng and Stefano Ermon},
booktitle={International Conference on Learning Representations},
year={2021},
url={https://openreview.net/forum?id=St1giarCHLP}
}

@inproceedings{i2sb,
    author = {Liu, Guan-Horng and Vahdat, Arash and Huang, De-An and Theodorou, Evangelos A. and Nie, Weili and Anandkumar, Anima},
    title = {I2SB: image-to-image Schr\"{o}dinger bridge},
    year = {2023},
    publisher = {JMLR.org},
    booktitle = {Proceedings of the 40th International Conference on Machine Learning},
    articleno = {915},
    numpages = {21},
    location = {Honolulu, Hawaii, USA},
    series = {ICML'23}
}

@InProceedings{pmlr-v37-sohl-dickstein15,
  title = 	 {Deep Unsupervised Learning using Nonequilibrium Thermodynamics},
  author = 	 {Sohl-Dickstein, Jascha and Weiss, Eric and Maheswaranathan, Niru and Ganguli, Surya},
  booktitle = 	 {Proceedings of the 32nd International Conference on Machine Learning},
  pages = 	 {2256--2265},
  year = 	 {2015},
  editor = 	 {Bach, Francis and Blei, David},
  volume = 	 {37},
  series = 	 {Proceedings of Machine Learning Research},
  address = 	 {Lille, France},
  month = 	 {07--09 Jul},
  publisher =    {PMLR},
  url = 	 {https://proceedings.mlr.press/v37/sohl-dickstein15.html}
}

@inproceedings{
wang2025framebridge,
title={FrameBridge: Improving Image-to-Video Generation with Bridge Models},
author={Wang, Yuji and Chen, Zehua and Chen, Xiaoyu and Wei, Yixiang and Zhu, Jun and Chen, Jianfei},
booktitle={Forty-second International Conference on Machine Learning},
year={2025},
url={https://openreview.net/forum?id=iYmV2xRSNW}
}

@misc{delbracio2024inversiondirectiterationalternative,
      title={Inversion by Direct Iteration: An Alternative to Denoising Diffusion for Image Restoration}, 
      author={Mauricio Delbracio and Peyman Milanfar},
      year={2024},
      eprint={2303.11435},
      archivePrefix={arXiv},
      primaryClass={eess.IV},
      url={https://arxiv.org/abs/2303.11435}, 
}

@inproceedings{Niklaus_CVPR_2020,
    author = {Simon Niklaus and Feng Liu},
    title = {Softmax Splatting for Video Frame Interpolation},
    booktitle = {IEEE Conference on Computer Vision and Pattern Recognition},
    year = {2020}
}

@inproceedings{huang2022rife,
  title={Real-Time Intermediate Flow Estimation for Video Frame Interpolation},
  author={Huang, Zhewei and Zhang, Tianyuan and Heng, Wen and Shi, Boxin and Zhou, Shuchang},
  booktitle={Proceedings of the European Conference on Computer Vision (ECCV)},
  year={2022}
}

@inproceedings{park2021ABME,
    author    = {Park, Junheum and Lee, Chul and Kim, Chang-Su}, 
    title     = {Asymmetric Bilateral Motion Estimation for Video Frame Interpolation}, 
    booktitle = {International Conference on Computer Vision},
    year      = {2021}
}

@inproceedings{lee2020adacof,
    title={AdaCoF: Adaptive Collaboration of Flows for Video Frame Interpolation},
    author={Lee, Hyeongmin and Kim, Taeoh and Chung, Tae-young and Pak, Daehyun and Ban, Yuseok and Lee, Sangyoun},
    booktitle={Proceedings of the IEEE/CVF Conference on Computer Vision and Pattern Recognition (CVPR)},
    year={2020}
}

@INPROCEEDINGS{10030268,
  author={Kalluri, Tarun and Pathak, Deepak and Chandraker, Manmohan and Tran, Du},
  booktitle={2023 IEEE/CVF Winter Conference on Applications of Computer Vision (WACV)}, 
  title={FLAVR: Flow-Agnostic Video Representations for Fast Frame Interpolation}, 
  year={2023},
  volume={},
  number={},
  pages={2070-2081},
  doi={10.1109/WACV56688.2023.00211}
}

@inproceedings{shi2022video,
  title={Video Frame Interpolation Transformer},
  author={Shi, Zhihao and Xu, Xiangyu and Liu, Xiaohong and Chen, Jun and Yang, Ming-Hsuan},
  booktitle={CVPR},
  year={2022}
}

@inproceedings{reda2022film,
 title = {FILM: Frame Interpolation for Large Motion},
 author = {Fitsum Reda and Janne Kontkanen and Eric Tabellion and Deqing Sun and Caroline Pantofaru and Brian Curless},
 booktitle = {European Conference on Computer Vision (ECCV)},
 year = {2022}
}

@inproceedings{voleti2022MCVD,
 author = {Voleti, Vikram and Jolicoeur-Martineau, Alexia and Pal, Christopher},
 title = {MCVD: Masked Conditional Video Diffusion for Prediction, Generation, and Interpolation},
 url = {https://arxiv.org/abs/2205.09853},
 booktitle = {(NeurIPS) Advances in Neural Information Processing Systems},
 year = {2022}
}

@inproceedings{10.1609/aaai.v38i2.27912,
author = {Danier, Duolikun and Zhang, Fan and Bull, David},
title = {LDMVFI: video frame interpolation with latent diffusion models},
year = {2024},
isbn = {978-1-57735-887-9},
publisher = {AAAI Press},
url = {https://doi.org/10.1609/aaai.v38i2.27912},
doi = {10.1609/aaai.v38i2.27912},
booktitle = {Proceedings of the Thirty-Eighth AAAI Conference on Artificial Intelligence and Thirty-Sixth Conference on Innovative Applications of Artificial Intelligence and Fourteenth Symposium on Educational Advances in Artificial Intelligence},
articleno = {164},
numpages = {9},
series = {AAAI'24/IAAI'24/EAAI'24}
}

@misc{jain2024videointerpolationdiffusionmodels,
      title={Video Interpolation with Diffusion Models}, 
      author={Siddhant Jain and Daniel Watson and Eric Tabellion and Aleksander Hołyński and Ben Poole and Janne Kontkanen},
      year={2024},
      eprint={2404.01203},
      archivePrefix={arXiv},
      primaryClass={cs.CV},
      url={https://arxiv.org/abs/2404.01203}, 
}

@inproceedings{wang2025generative,
title={Generative Inbetweening: Adapting Image-to-Video Models for Keyframe Interpolation},
author={Xiaojuan Wang and Boyang Zhou and Brian Curless and Ira Kemelmacher-Shlizerman and Aleksander Holynski and Steve Seitz},
booktitle={The Thirteenth International Conference on Learning Representations},
year={2025},
url={https://openreview.net/forum?id=ykD8a9gJvy}
}

@inproceedings{10.1007/978-3-031-72633-0_19,
author = {Shen, Liao and Liu, Tianqi and Sun, Huiqiang and Ye, Xinyi and Li, Baopu and Zhang, Jianming and Cao, Zhiguo},
title = {DreamMover: Leveraging the Prior of Diffusion Models for Image Interpolation with Large Motion},
year = {2024},
isbn = {978-3-031-72632-3},
publisher = {Springer-Verlag},
address = {Berlin, Heidelberg},
url = {https://doi.org/10.1007/978-3-031-72633-0_19},
doi = {10.1007/978-3-031-72633-0_19},
booktitle = {Computer Vision – ECCV 2024: 18th European Conference, Milan, Italy, September 29–October 4, 2024, Proceedings, Part XV},
pages = {336–353},
numpages = {18},
location = {Milan, Italy}
}

@misc{chen2025repurposingpretrainedvideodiffusion,
      title={Repurposing Pre-trained Video Diffusion Models for Event-based Video Interpolation}, 
      author={Jingxi Chen and Brandon Y. Feng and Haoming Cai and Tianfu Wang and Levi Burner and Dehao Yuan and Cornelia Fermuller and Christopher A. Metzler and Yiannis Aloimonos},
      year={2025},
      eprint={2412.07761},
      archivePrefix={arXiv},
      primaryClass={cs.CV},
      url={https://arxiv.org/abs/2412.07761}, 
}

@misc{zhang2025motionawaregenerativeframeinterpolation,
      title={Motion-Aware Generative Frame Interpolation}, 
      author={Guozhen Zhang and Yuhan Zhu and Yutao Cui and Xiaotong Zhao and Kai Ma and Limin Wang},
      year={2025},
      eprint={2501.03699},
      archivePrefix={arXiv},
      primaryClass={cs.CV},
      url={https://arxiv.org/abs/2501.03699}, 
}

@article{DBLP:journals/corr/SrivastavaMS15,
  author    = {Nitish Srivastava and
               Elman Mansimov and
               Ruslan Salakhutdinov},
  title     = {Unsupervised Learning of Video Representations using LSTMs},
  journal   = {CoRR},
  volume    = {abs/1502.04681},
  year      = {2015},
  url       = {http://arxiv.org/abs/1502.04681},
  archivePrefix = {arXiv},
  eprint    = {1502.04681},
  bibsource = {dblp computer science bibliography, https://dblp.org}
}

@misc{xing2023dynamicrafter,
      title={DynamiCrafter: Animating Open-domain Images with Video Diffusion Priors}, 
      author={Jinbo Xing and Menghan Xia and Yong Zhang and Haoxin Chen and Xintao Wang and Tien-Tsin Wong and Ying Shan},
      year={2023},
      eprint={2310.12190},
      archivePrefix={arXiv},
      primaryClass={cs.CV}
}

@misc{2023i2vgenxl,
title={I2VGen-XL: High-Quality Image-to-Video Synthesis via Cascaded Diffusion Models},
 author={Zhang, Shiwei* and Wang, Jiayu* and Zhang, Yingya* and Zhao, Kang and Yuan, Hangjie and Qing, Zhiwu and Wang, Xiang and Zhao, Deli and Zhou, Jingren},
 booktitle={arXiv preprint arXiv:2311.04145},
 year={2023}
}

@article{shi2024motion,
            title={Motion-i2v: Consistent and controllable image-to-video generation with explicit motion modeling},
            author={Shi, Xiaoyu and Huang, Zhaoyang and Wang, Fu-Yun and Bian, Weikang and Li, Dasong and Zhang, Yi and Zhang, Manyuan and Cheung, Ka Chun and See, Simon and Qin, Hongwei and others},
            journal={SIGGRAPH 2024},
            year={2024}
}

@misc{guo2023animatediff,
  title={AnimateDiff: Animate Your Personalized Text-to-Image Diffusion Models without Specific Tuning},
  author={Yuwei Guo and Ceyuan Yang and Anyi Rao and Zhengyang Liang and Yaohui Wang and Yu Qiao and Maneesh Agrawala and Dahua Lin and Bo Dai},
  booktitle={arXiv preprint arxiv:2307.04725},
  year={2023},
  archivePrefix={arXiv},
  primaryClass={cs.CV}
}

@misc{guo2024i2vadaptergeneralimagetovideoadapter,
      title={I2V-Adapter: A General Image-to-Video Adapter for Diffusion Models}, 
      author={Xun Guo and Mingwu Zheng and Liang Hou and Yuan Gao and Yufan Deng and Pengfei Wan and Di Zhang and Yufan Liu and Weiming Hu and Zhengjun Zha and Haibin Huang and Chongyang Ma},
      year={2024},
      eprint={2312.16693},
      archivePrefix={arXiv},
      primaryClass={cs.CV},
      url={https://arxiv.org/abs/2312.16693}, 
}

@article{wu2023freeinit,
     title={FreeInit: Bridging Initialization Gap in Video Diffusion Models},
     author={Wu, Tianxing and Si, Chenyang and Jiang, Yuming and Huang, Ziqi and Liu, Ziwei},
     journal={arXiv preprint arXiv:2312.07537},
     year={2023}
}

@article{ren2024consisti2v,
  title={ConsistI2V: Enhancing Visual Consistency for Image-to-Video Generation},
  author={Ren, Weiming and Yang, Harry and Zhang, Ge and Wei, Cong and Du, Xinrun and Huang, Stephen and Chen, Wenhu},
  journal={arXiv preprint arXiv:2402.04324},
  year={2024}
}

@inproceedings{
    loshchilov2018decoupled,
    title={Decoupled Weight Decay Regularization},
    author={Ilya Loshchilov and Frank Hutter},
    booktitle={International Conference on Learning Representations},
    year={2019},
    url={https://openreview.net/forum?id=Bkg6RiCqY7},
}

@misc{unterthiner2019accurategenerativemodelsvideo,
      title={Towards Accurate Generative Models of Video: A New Metric \& Challenges}, 
      author={Thomas Unterthiner and Sjoerd van Steenkiste and Karol Kurach and Raphael Marinier and Marcin Michalski and Sylvain Gelly},
      year={2019},
      eprint={1812.01717},
      archivePrefix={arXiv},
      primaryClass={cs.CV},
      url={https://arxiv.org/abs/1812.01717}, 
}

@InProceedings{zhou2024upscaleavideo,
      title     = {{Upscale-A-Video}: Temporal-Consistent Diffusion Model for Real-World Video Super-Resolution},
      author    = {Zhou, Shangchen and Yang, Peiqing and Wang, Jianyi and Luo, Yihang and Loy, Chen Change},
      booktitle = {CVPR},
      year      = {2024}
    }

@misc{yang2024motionguidedlatentdiffusiontemporally,
      title={Motion-Guided Latent Diffusion for Temporally Consistent Real-world Video Super-resolution}, 
      author={Xi Yang and Chenhang He and Jianqi Ma and Lei Zhang},
      year={2024},
      eprint={2312.00853},
      archivePrefix={arXiv},
      primaryClass={cs.CV},
      url={https://arxiv.org/abs/2312.00853}, 
}

@misc{gushchin2025inversebridgematchingdistillation,
      title={Inverse Bridge Matching Distillation}, 
      author={Nikita Gushchin and David Li and Daniil Selikhanovych and Evgeny Burnaev and Dmitry Baranchuk and Alexander Korotin},
      year={2025},
      eprint={2502.01362},
      archivePrefix={arXiv},
      primaryClass={cs.LG},
      url={https://arxiv.org/abs/2502.01362}, 
}

@inproceedings{blattmann2023videoldm,
    title={Align your Latents: High-Resolution Video Synthesis with Latent Diffusion Models},
    author={Blattmann, Andreas and Rombach, Robin and Ling, Huan and Dockhorn, Tim and Kim, Seung Wook and Fidler, Sanja and Kreis, Karsten},
    booktitle={IEEE Conference on Computer Vision and Pattern Recognition ({CVPR})},
    year={2023}
}

@misc{ho2022imagenvideohighdefinition,
      title={Imagen Video: High Definition Video Generation with Diffusion Models}, 
      author={Jonathan Ho and William Chan and Chitwan Saharia and Jay Whang and Ruiqi Gao and Alexey Gritsenko and Diederik P. Kingma and Ben Poole and Mohammad Norouzi and David J. Fleet and Tim Salimans},
      year={2022},
      eprint={2210.02303},
      archivePrefix={arXiv},
      primaryClass={cs.CV},
      url={https://arxiv.org/abs/2210.02303}, 
}

@misc{arkhipkin2023fusionframesefficientarchitecturalaspects,
      title={FusionFrames: Efficient Architectural Aspects for Text-to-Video Generation Pipeline}, 
      author={Vladimir Arkhipkin and Zein Shaheen and Viacheslav Vasilev and Elizaveta Dakhova and Andrey Kuznetsov and Denis Dimitrov},
      year={2023},
      eprint={2311.13073},
      archivePrefix={arXiv},
      primaryClass={cs.CV},
      url={https://arxiv.org/abs/2311.13073}, 
}

@article{ma2025latte,
  title={Latte: Latent Diffusion Transformer for Video Generation},
  author={Ma, Xin and Wang, Yaohui and Chen, Xinyuan and Jia, Gengyun and Liu, Ziwei and Li, Yuan-Fang and Chen, Cunjian and Qiao, Yu},
  journal={Transactions on Machine Learning Research},
  year={2025}
}

@inproceedings{zhang2025fast,
    title={Fast Video Generation with Sliding Tile Attention},
    author={Peiyuan Zhang and Yongqi Chen and Runlong Su and Hangliang Ding and Ion Stoica and Zhengzhong Liu and Hao Zhang},
    booktitle={Forty-second International Conference on Machine Learning},
    year={2025},
    url={https://openreview.net/forum?id=U74MOXPEJd}
}

@misc{xi2025sparsevideogenacceleratingvideo,
      title={Sparse VideoGen: Accelerating Video Diffusion Transformers with Spatial-Temporal Sparsity}, 
      author={Haocheng Xi and Shuo Yang and Yilong Zhao and Chenfeng Xu and Muyang Li and Xiuyu Li and Yujun Lin and Han Cai and Jintao Zhang and Dacheng Li and Jianfei Chen and Ion Stoica and Kurt Keutzer and Song Han},
      year={2025},
      eprint={2502.01776},
      archivePrefix={arXiv},
      primaryClass={cs.CV},
      url={https://arxiv.org/abs/2502.01776}, 
}

@ARTICLE{mikhailov2025nablanablaneighborhoodadaptiveblocklevel,
  author={Mikhailov, Dmitrii and Letunovskiy, Aleksey and Kovaleva, Maria and Arkhipkin, Vladimir and Korviakov, Vladimir and Polovnikov, Vladimir and Vasilev, Viacheslav and Sidorova, Evelina and Dimitrov, Denis},
  journal={IEEE Access}, 
  title={NABLA: Neighborhood Adaptive Block-Level Attention for Efficient Video Generation}, 
  year={2026},
  volume={14},
  number={},
  pages={64655-64665},
  doi={10.1109/ACCESS.2026.3686867}
}

@misc{soomro2012ucf101dataset101human,
      title={UCF101: A Dataset of 101 Human Actions Classes From Videos in The Wild}, 
      author={Khurram Soomro and Amir Roshan Zamir and Mubarak Shah},
      year={2012},
      eprint={1212.0402},
      archivePrefix={arXiv},
      primaryClass={cs.CV},
      url={https://arxiv.org/abs/1212.0402}, 
}

@misc{saito2017temporalgenerativeadversarialnets,
      title={Temporal Generative Adversarial Nets with Singular Value Clipping}, 
      author={Masaki Saito and Eiichi Matsumoto and Shunta Saito},
      year={2017},
      eprint={1611.06624},
      archivePrefix={arXiv},
      primaryClass={cs.LG},
      url={https://arxiv.org/abs/1611.06624}, 
}

@misc{arkhipkin2025kandinsky50familyfoundation,
      title={Kandinsky 5.0: A Family of Foundation Models for Image and Video Generation}, 
      author={Vladimir Arkhipkin and Vladimir Korviakov and Nikolai Gerasimenko and Denis Parkhomenko and Viacheslav Vasilev and Alexey Letunovskiy and Nikolai Vaulin and Maria Kovaleva and Ivan Kirillov and Lev Novitskiy and Denis Koposov and Nikita Kiselev and Alexander Varlamov and Dmitrii Mikhailov and Vladimir Polovnikov and Andrey Shutkin and Julia Agafonova and Ilya Vasiliev and Anastasiia Kargapoltseva and Anna Dmitrienko and Anastasia Maltseva and Anna Averchenkova and Olga Kim and Tatiana Nikulina and Denis Dimitrov},
      year={2025},
      eprint={2511.14993},
      archivePrefix={arXiv},
      primaryClass={cs.CV},
      url={https://arxiv.org/abs/2511.14993}, 
}

@inproceedings{chang2025iwarpednoisetemporallycorrelated,
      title={{How I Warped Your Noise}: a Temporally-Correlated Noise Prior for Diffusion Models}, 
      author={Pascal Chang and Jingwei Tang and Markus Gross and Vinicius C. Azevedo},
      booktitle={International Conference on Learning Representations},
      year={2024},
      eprint={2504.03072},
      archivePrefix={arXiv},
      primaryClass={cs.CV},
      url={https://arxiv.org/abs/2504.03072}, 
}

@misc{liu2025equivariancefastsamplingvideo,
      title={On Equivariance and Fast Sampling in Video Diffusion Models Trained with Warped Noise}, 
      author={Chao Liu and Arash Vahdat},
      year={2025},
      eprint={2504.09789},
      archivePrefix={arXiv},
      primaryClass={cs.CV},
      url={https://arxiv.org/abs/2504.09789}, 
}

@misc{guo2023sparsectrladdingsparsecontrols,
      title={SparseCtrl: Adding Sparse Controls to Text-to-Video Diffusion Models}, 
      author={Yuwei Guo and Ceyuan Yang and Anyi Rao and Maneesh Agrawala and Dahua Lin and Bo Dai},
      year={2023},
      eprint={2311.16933},
      archivePrefix={arXiv},
      primaryClass={cs.CV},
      url={https://arxiv.org/abs/2311.16933}, 
}

@article{Xue_2019,
   title={Video Enhancement with Task-Oriented Flow},
   volume={127},
   ISSN={1573-1405},
   url={http://dx.doi.org/10.1007/s11263-018-01144-2},
   DOI={10.1007/s11263-018-01144-2},
   number={8},
   journal={International Journal of Computer Vision},
   publisher={Springer Science and Business Media LLC},
   author={Xue, Tianfan and Chen, Baian and Wu, Jiajun and Wei, Donglai and Freeman, William T.},
   year={2019},
   month=Feb, pages={1106–1125}
}

@INPROCEEDINGS{flolpips,
  author={Danier, Duolikun and Zhang, Fan and Bull, David},
  booktitle={2022 Picture Coding Symposium (PCS)}, 
  title={FloLPIPS: A Bespoke Video Quality Metric for Frame Interpolation}, 
  year={2022},
  volume={},
  number={},
  pages={283-287},
  doi={10.1109/PCS56426.2022.10018062}}

@inproceedings{lyu2024cbbd,
  title={Frame Interpolation with Consecutive {B}rownian Bridge Diffusion},
  author={Lyu, Zonglin and Li, Ming and Jiao, Jianbo and Chen, Chen},
  booktitle={Proceedings of the 32nd ACM International Conference on Multimedia},
  pages={3449--3458},
  year={2024}
}

@inproceedings{lyu2025tlbvfi,
  title={{TLB-VFI}: Temporal-Aware Latent {B}rownian Bridge Diffusion for Video Frame Interpolation},
  author={Lyu, Zonglin and Chen, Chen},
  booktitle={Proceedings of the IEEE/CVF International Conference on Computer Vision (ICCV)},
  year={2025}
}

@inproceedings{yue2024goub,
  title={Image Restoration Through Generalized {O}rnstein-{U}hlenbeck Bridge},
  author={Yue, Conghan and Peng, Zhengwei and Ma, Junlong and Du, Shiyan and Wei, Pengxu and Zhang, Dongyu},
  booktitle={Proceedings of the 41st International Conference on Machine Learning},
  series={Proceedings of Machine Learning Research},
  volume={235},
  pages={58068--58089},
  year={2024}
}

@inproceedings{hoogeboom2023blurring,
  title={Blurring Diffusion Models},
  author={Hoogeboom, Emiel and Salimans, Tim},
  booktitle={The Eleventh International Conference on Learning Representations},
  year={2023}
}

@inproceedings{rissanen2023ihdm,
  title={Generative Modelling with Inverse Heat Dissipation},
  author={Rissanen, Severi and Heinonen, Markus and Solin, Arno},
  booktitle={The Eleventh International Conference on Learning Representations},
  year={2023}
}

@book{strang2007computational,
  title={Computational Science and Engineering},
  author={Strang, Gilbert},
  publisher={Wellesley-Cambridge Press},
  year={2007}
}

@book{sarkka2019applied,
  title={Applied Stochastic Differential Equations},
  author={S{\"a}rkk{\"a}, Simo and Solin, Arno},
  publisher={Cambridge University Press},
  year={2019}
}

@article{park2024timebridge,
  title={{TimeBridge}: Better Diffusion Prior Design with Bridge Models for Time Series Generation},
  author={Park, Jinseong and Lee, Seungyun and Jeong, Woojin and Choi, Yujin and Lee, Jaewook},
  journal={arXiv preprint arXiv:2408.06672},
  year={2024}
}

@inproceedings{chen2023csbi,
  title={Provably Convergent {S}chr{\"o}dinger Bridge with Applications to Probabilistic Time Series Imputation},
  author={Chen, Yu and Deng, Wei and Fang, Shikai and Li, Fengpei and Yang, Nicole Tianjiao and Zhang, Yikai and Rasul, Kashif and Zhe, Shandian and Schneider, Anderson and Nevmyvaka, Yuriy},
  booktitle={Proceedings of the 40th International Conference on Machine Learning},
  series={Proceedings of Machine Learning Research},
  volume={202},
  pages={4485--4513},
  year={2023}
}

@inproceedings{skorokhodov2022styleganv,
  title={{StyleGAN-V}: A Continuous Video Generator with the Price, Image Quality and Perks of {StyleGAN2}},
  author={Skorokhodov, Ivan and Tulyakov, Sergey and Elhoseiny, Mohamed},
  booktitle={Proceedings of the IEEE/CVF Conference on Computer Vision and Pattern Recognition (CVPR)},
  year={2022}
}

@book{rue2005gmrf,
  title={Gaussian {M}arkov Random Fields: Theory and Applications},
  author={Rue, H{\aa}vard and Held, Leonhard},
  publisher={Chapman and Hall/CRC},
  year={2005}
}

\clearpage
\appendix
\thispagestyle{empty}

\onecolumn
\aistatstitle{Time-Correlated Video Bridge Matching: \\
Supplementary Materials}

\section{Additional Related Work}\label{sec:additional_related}

The appendix is organized as follows: Appendix~\ref{sec:additional_related} reviews task-specific related work, \ref{appendix:proofs} contains the proofs, \ref{appendix:spectral} the spectral analysis, \ref{sec:mnist_implementation} implementation details, \ref{appendix:metrics} standard deviations and the Moving MNIST user study, \ref{appendix:initialization} initialization experiments, \ref{appendix:hyperparameters} the hyperparameter study, \ref{appendix:comp_cost} the computational cost, and \ref{appendix:dynamical_correlation} a time-dependent coupling.

\paragraph{Frame Interpolation} aims to synthesize middle frames from two inputs, ensuring smoothness and consistency. Traditional methods use optical flow~\citep{Niklaus_CVPR_2020, lee2020adacof, park2021ABME, huang2022rife} or feature-based and transformer-based architectures~\citep{10030268, shi2022video, reda2022film}. Diffusion-based interpolation began with bidirectional masking~\citep{voleti2022MCVD} and advanced through conditional generation~\citep{10.1609/aaai.v38i2.27912}, cascaded refinement~\citep{jain2024videointerpolationdiffusionmodels}, adapted image-to-video models~\citep{wang2025generative}, and large-motion techniques~\citep{10.1007/978-3-031-72633-0_19}. However, these methods do not model inter-frame correlations in the prior process of the generative model. Other approaches add explicit motion cues, such as event-camera data~\citep{chen2025repurposingpretrainedvideodiffusion} or motion guidance~\citep{zhang2025motionawaregenerativeframeinterpolation}, but also use standard diffusion priors.

\paragraph{Image-to-Video Generation} creates a video from an input image, requiring consistent and accurate motion. Diffusion models have significantly advanced this field, producing high-quality results~\citep{2023i2vgenxl, xing2023dynamicrafter, shi2024motion, guo2023animatediff, guo2024i2vadaptergeneralimagetovideoadapter, 10815947, arkhipkin2025kandinsky50familyfoundation}. However, the standard noise-to-data diffusion process risks losing essential information from the input image. ConsistI2V~\citep{ren2024consisti2v} addresses this within the diffusion framework, and FreeInit~\citep{wu2023freeinit} improves the noise initialization of video diffusion models. 
A recent extension of bridge matching to image-to-video generation, although taking into account temporal correlation between frames, did not modify the interpolant itself~\cite{wang2025framebridge}. In our solution, we address this gap and develop a new prior that explicitly takes into account the correlations between the components of random vectors representing video data samples.

\paragraph{Video Super-Resolution} reconstructs high-resolution videos from low-resolution inputs. Diffusion models are applied for their strong generative prior, which synthesizes realistic details to overcome degradation. A key challenge is ensuring temporal coherence within the inherently stochastic diffusion process. Recent methods address this by introducing explicit spatiotemporal constraints, such as temporal layer integration and motion-guided losses~\citep{zhou2024upscaleavideo, yang2024motionguidedlatentdiffusiontemporally}. Meanwhile, bridge-matching methods have shown promise for image super-resolution~\citep{i2sb, gushchin2025inversebridgematchingdistillation}. In this work, we extend bridge matching to video super-resolution, proposing a novel approach that explicitly models temporal coherence between frames.

\section{Proofs}\label{appendix:proofs}

\begin{proof}[Proof of Proposition 1]
Consider the linear SDE
$$
    d\X_t = (\A \X_t + \bv)dt + \sqrt{\epsilon}\, d\W_t, \quad \X_0 \sim \delta_{\X_0},
$$
with $\A\in\mathbb R^{N\times N}$ symmetric and invertible. Since the $D$ columns of $\X_t$ evolve independently with the same $\A$, it suffices to consider a single column; below, $\X_t, \bv \in\mathbb R^N$ and $\W_t$ is an $N$-dimensional standard Wiener process.

\textbf{Conditional mean.}
Let $\bPhi(t):=e^{\A t}$ and define $\Y_t:=(\bPhi(t))^{-1}\X_t=e^{-\A t}\X_t$ (note $\Y_0 = \X_0$).
\begin{gather}
    d\Y_t = d\big(e^{-\A t}\X_t\big) = e^{-\A t}\,d\X_t + d(e^{-\A t})\,\X_t = 
    \nonumber
    \\
    e^{-\A t}\big[(\A \X_t+\bv)\,dt+\sqrt{\epsilon}\,d\W_t\big] - \A e^{-\A t}\X_t\,dt 
    = e^{-\A t}\bv\,dt + \sqrt{\epsilon}\,e^{-\A t}d\W_t,
    \nonumber
\end{gather}
In integral form,
\begin{gather}
    \Y_t = \X_0 + \int_0^t e^{-\A s}\bv\,ds + \sqrt{\epsilon}\int_0^t e^{-\A s}\,d\W_s.
    \nonumber
\end{gather}
Multiplying by $\bPhi(t)=e^{\A t}$ yields:
\begin{gather}
    \X_t = e^{\A t}\X_0 + \int_0^t e^{\A(t-s)}\bv ds + \sqrt{\epsilon}\int_0^t e^{\A(t-s)}d\W_s. 
    \label{eq:variation-of-constants}
\end{gather}
Hence, the conditional mean is:
\begin{gather}
    \muvec_{t|0}(\X_0)=e^{\A t}\X_0+\Big(\!\int_0^t e^{\A(t-s)}ds\Big)\bv
= e^{\A t}\X_0+\big(e^{\A t}-\I\big)\A^{-1}\bv, \label{eq:conditional-mean}
\end{gather}
\textbf{Conditional variance.} 
$$
    \mathbf{J}_t:=\int_0^t e^{\A(t-s)}\,d\W_s,
    \qquad\text{so that}\qquad
    \X_t-\muvec_{t|0}(\X_0)=\sqrt{\epsilon}\,\mathbf{J}_t.
$$
$$
    \mathrm{Cov}(\X_t\mid \X_0)
= \mathbb E\big[(\X_t-\muvec_{t|0})(\X_t-\muvec_{t|0})^\top\mid \X_0\big]
= \epsilon\,\mathrm{Cov}(\mathbf{J}_t).
$$
In turn:
$$
    \mathrm{Cov}(\mathbf{J}_t)
=\mathbb E\!\left[\left(\int_{0}^{t} e^{\A(t-s)}\,d\W_s\right)
\left(\int_{0}^{t} e^{\A(t-r)}\,d\W_r\right)^{\!\top}\right].
$$
By Itô isometry:
$$
    \mathbb E\!\left[\int_0^t G_s\,d\W_s\right]=0,\qquad
\mathrm{Cov}\!\left(\int_0^t G_s\,d\W_s,\int_0^t H_s\,d\W_s\right)=\int_0^t G_s H_s^\top\,ds.
$$
Taking $G_s=H_s=e^{\A(t-s)}$ gives
$$
    \mathrm{Cov}(\mathbf{J}_t)=\int_{0}^{t} e^{\A(t-s)}\,e^{\A^\top(t-s)}\,ds.
$$
Because $\A$ is symmetric, $e^{\A^\top u}=e^{\A u}$, hence
$$
    \mathrm{Cov}(\mathbf{J}_t)=\int_{0}^{t} e^{2\A(t-s)}\,ds
=\frac{1}{2}\big(e^{2\A t}-\I\big)\A^{-1}.
$$
Consequently,
$$
    \Si_{t|0}=\mathrm{Cov}(\X_t\mid \X_0)
=\frac{\epsilon}{2}\big(e^{2\A t}-\I\big)\A^{-1}.
$$
Since $\X_t\mid \X_0$ is Gaussian with mean $\muvec_{t|0}$ and covariance $\Si_{t|0}$, its score is
$$
    \nabla_{\X_t}\log q(\X_t\mid \X_0)=-\,\Si_{t|0}^{-1}\big(\X_t-\muvec_{t|0}(\X_0)\big).
$$
This completes the proof.
\end{proof}

\begin{proof}[Proof of Proposition 2]
\textbf{Step 1: Joint law from the prior.}
From \eqref{eq:variation-of-constants} and \eqref{eq:conditional-mean} in the proof of Proposition 1: %
$$
    \X_u=\muvec_{u|0}(\X_0)+\sqrt{\epsilon}\int_0^u e^{\A(u-s)}d\W_s,
$$
$$
    \muvec_{u|0}(\X_0)=e^{\A u}\X_0+\big(e^{\A u}-\I\big)\A^{-1}\bv.
$$
Thus, conditionally on $\X_0$,
$$
    \mathbb E[\X_t\mid \X_0]=\muvec_{t|0}(\X_0),\qquad \mathbb E[\X_{t'}\mid \X_0]=\muvec_{t'|0}(\X_0),
$$
$$
    \Si_{t|0}=\frac{\epsilon}{2}\big(e^{2\A t}-\I\big)\A^{-1}, \qquad \Si_{t'|0}=\frac{\epsilon}{2}\big(e^{2\A t'}-\I\big)\A^{-1}.
$$
For the cross-covariance, using Itô isometry and independence of increments, for $t<t'$,
\begin{gather}
    \Si_{t,t'|0} =\mathrm{Cov}(\X_t,\X_{t'}\mid \X_0) =\epsilon \int_0^t e^{\A(t-s)}\,e^{\A(t'-s)}\,ds
    =\epsilon \int_0^t e^{\A(t+t'-2s)}\,ds
    \nonumber\\
    = \frac{\epsilon}{2}\,\A^{-1}\!\left(e^{\A(t+t')}-e^{\A(t'-t)}\right) = e^{\A(t'-t)}\,\Si_{t|0}.
    \nonumber
\end{gather}
Collecting blocks, we have the joint Gaussian (conditionally on $\X_0$)
\begin{gather}
    \begin{bmatrix}\X_t\\ \X_{t'}\end{bmatrix}
    \sim \mathcal N\!\left(
    \begin{bmatrix}\muvec_{t|0}(\X_0)\\ \muvec_{t'|0}(\X_0)\end{bmatrix},
    \begin{bmatrix}
    \Si_{t|0} & \Si_{t,t'|0}\\
    \Si_{t,t'|0} & \Si_{t'|0}
    \end{bmatrix}
    \right).
    \nonumber
\end{gather}

\textbf{Step 2: Conditioning to obtain the bridge.}
For a joint Gaussian $\begin{bmatrix}x\\y\end{bmatrix}$ with blocks $(\mu_x,\mu_y,\Sigma_{xx},\Sigma_{yy},\Sigma_{xy})$, the conditional $x\mid y$ is \citep[Section 8.1.3]{petersen2008matrix}:
$$
    x\mid y \sim \mathcal N\big(\mu_x+\Sigma_{xy}\Sigma_{yy}^{-1}(y-\mu_y),
\Sigma_{xx}-\Sigma_{xy}\Sigma_{yy}^{-1}\Sigma_{yx}\big).
$$
Applying this with $x=\X_t$, $y=\X_{t'}$ and the blocks above gives:
\begin{gather}
    \muvec_{t\mid 0,t'} = \muvec_{t|0}(\X_0)+\Si_{t,t'|0}\,\Si_{t'|0}^{-1}\!\big(\X_{t'}-\muvec_{t'|0}(\X_0)\big),
    \nonumber
    \\
    \nonumber
\Si_{t\mid 0,t'} = \Si_{t|0}-\Si_{t,t'|0}\,\Si_{t'|0}^{-1}\Si_{t,t'|0}.
\nonumber
\end{gather}
Here
\begin{gather}
    \Si_{t|0}=\tfrac{\epsilon}{2}(e^{2\A t}-\I)\A^{-1},\quad
    \Si_{t'|0}=\tfrac{\epsilon}{2}(e^{2\A t'}-\I)\A^{-1},\quad
    \Si_{t,t'|0}=\tfrac{\epsilon}{2}\A^{-1}\big(e^{\A(t+t')}-e^{\A(t'-t)}\big)=e^{\A(t'-t)}\,\Si_{t|0}.
    \nonumber
\end{gather}
Since all these matrices are functions of $\A$, they are symmetric and commute, so $\Si_{t,t'|0}^\top=\Si_{t,t'|0}$; this gives \eqref{eq:bridge-mean}--\eqref{eq:bridge-cov}.
This completes the proof.
\end{proof}

\begin{proof}[Proof of Proposition 3]
    Consider the following bijective reparameterization:
    $$
        v_{\phi}(\X_t, t) =  -\Si_{t|0}^{-1} \left(\X_t - \muvec_{t|0}(\widehat{\X}_0^{\phi}(\X_t, t))\right)
    $$
    and substitute it in the optimization problem:
    \begin{equation}   
        \min_{\phi} \mathbb{E}_{\X_0, \X_t, t} \left[ \left\| v_\phi(\X_t, t) + \Si_{t|0}^{-1} (\X_t - \muvec_{t|0}(\X_0)) \right\|^2 \right],
        \nonumber
    \end{equation}
    \begin{gather}
        \min_{\phi} \mathbb{E}_{\X_0, \X_t, t} \left[ \left\| -\Si_{t|0}^{-1} \left(\X_t - \muvec_{t|0}(\widehat{\X}_0^{\phi}(\X_t, t))\right) + \Si_{t|0}^{-1} (\X_t - \muvec_{t|0}(\X_0)) \right\|^2 \right] = 
        \nonumber
        \\
        \min_{\phi} \mathbb{E}_{\X_0, \X_t, t} \left[ \left\| \Si_{t|0}^{-1} \left((\X_t - \muvec_{t|0}(\X_0))  -\left(\X_t - \muvec_{t|0}(\widehat{\X}_0^{\phi}(\X_t, t))\right)  \right) \right\|^2 \right] =
        \nonumber
        \\
        \min_{\phi} \mathbb{E}_{\X_0, \X_t, t} \left[ \left\| \Si_{t|0}^{-1} \left( \muvec_{t|0}(\widehat{\X}_0^{\phi}(\X_t, t)) - \muvec_{t|0}(\X_0) \right) \right\|^2 \right] =
        \nonumber
        \\
        \min_{\phi} \mathbb{E}_{\X_0, \X_t, t} \left[\left(\muvec_{t|0}(\widehat{\X}_0^{\phi}(\X_t, t)) - \muvec_{t|0}(\X_0)\right)^{\top} (\Si_{t|0}^{\top})^{-1} \Si_{t|0}^{-1} \left(\muvec_{t|0}(\widehat{\X}_0^{\phi}(\X_t, t)) - \muvec_{t|0}(\X_0) \right) \right]
        \nonumber
    \end{gather}
    
    Since the minimization is over all measurable functions, it decouples over $(\X_t, t)$: for every fixed $(\X_t, t)$, we minimize the conditional expectation $\mathbb{E}_{\X_0|\X_t}[\,\cdot\,]$ of the last expression over the value $\mathbf{Y} = \widehat{\X}_0^{\phi}(\X_t, t)$ (column-wise; the squared norm is the Frobenius norm). Since $\muvec_{t|0}(\mathbf{Y}) = e^{\A t}\mathbf{Y} + (e^{\A t}-\I)\A^{-1}\bv$ has Jacobian $e^{\A t}$ with respect to $\mathbf{Y}$, setting the gradient with respect to $\mathbf{Y}$ to zero gives:
    \begin{gather}
        \mathbb{E}_{\X_0|\X_t} \left[2\, e^{\A t} (\Si_{t|0}^{\top})^{-1} \Si_{t|0}^{-1} \left(\muvec_{t|0}(\widehat{\X}_0^{\phi}(\X_t, t)) - \muvec_{t|0}(\X_0)  \right) \right] = 0
        \nonumber
    \end{gather}
    For fixed $t$, the matrices $e^{\A t}$ and $\Si_{t|0}$ are constant and invertible, so we can multiply by $\frac{1}{2}\Si_{t|0} \Si_{t|0}^{\top} e^{-\A t}$ and get:
    \begin{gather}
        \mathbb{E}_{\X_0|\X_t} \left[\muvec_{t|0}(\widehat{\X}_0^{\phi}(\X_t, t)) - \muvec_{t|0}(\X_0) \right] = 0
        \nonumber
    \end{gather}
    From \eqref{eq:conditional-mean}  we have:
    \begin{gather}
    \muvec_{t|0}(\X_0)=e^{\A t}\X_0+\Big(\!\int_0^t e^{\A(t-s)}ds\Big)\bv  = e^{\A t}\X_0+\big(e^{\A t}-\I\big)\A^{-1}\bv, 
    \nonumber
    \end{gather}
    Then (note that $e^{\A t}$ is invertible and we can multiply both sides by $e^{-\A t}$):
    \begin{gather}
        \mathbb{E}_{\X_0|\X_t} \left[ e^{\A t}\widehat{\X}_0^{\phi}(\X_t, t) +\big(e^{\A t}-\I\big)\A^{-1}\bv -  \Big(e^{\A t}\X_0+\big(e^{\A t}-\I\big)\A^{-1}\bv\Big) \right] = 0
        \nonumber
        \\
        \mathbb{E}_{\X_0|\X_t} \left[ e^{\A t}(\widehat{\X}_0^{\phi}(\X_t, t) - \X_0) \right] = 0
        \nonumber
        \\
        \mathbb{E}_{\X_0|\X_t} \left[\widehat{\X}_0^{\phi}(\X_t, t) - \X_0 \right] = 0
        \nonumber
        \\
        \widehat{\X}_0^{\phi}(\X_t, t) = \mathbb{E}_{\X_0|\X_t} \left[\X_0\right]
        \nonumber
    \end{gather}
    The objective is a convex quadratic in $\mathbf{Y}$ with a positive definite Hessian, so this stationary point is its unique minimizer.
    Hence, the optimal $\widehat{\X}_0^* = \mathbb{E}_{\X_0|\X_t} \left[\X_0\right]$, which in turn is the minimizer of the MSE problem over all measurable functions:
    \begin{gather}
        \min_{\phi} \mathbb{E}_{\X_0, \X_t, t}  \left[\|\widehat{\X}_0^{\phi}(\X_t, t) - \X_0\|^2 \right].
        \nonumber
    \end{gather}
    
    By substituting it into $v_{\phi}$ we have:
    \begin{gather}
        v^*(\X_t, t) =  -\Si_{t|0}^{-1} \left(\X_t - \muvec_{t|0}(\widehat{\X}_0^{*}(\X_t, t))\right)
        \nonumber
    \end{gather}
    This completes the proof. Note that, for a model class of finite capacity, \eqref{eq:initial-obj} and \eqref{eq:final-obj} differ by a $t$-dependent weighting of the error; the two problems share the same minimizer over all measurable functions, and we use \eqref{eq:final-obj} in practice.
\end{proof}

\section{Spectral View of the Time-Correlated Prior}\label{appendix:spectral}

Here we derive the statements of Section~\ref{sec:spectral}. We consider a single feature column (all columns share the same dynamics) and the prior $d\x_t = (\alpha\A\x_t + \bv)\,dt + \sqrt{\epsilon}\,d\W_t$ with the tridiagonal $\A$ of Section~\ref{sec:theory_tasks}; Propositions~\ref{thm:correlated-process-score}--\ref{thm:correlated-bridge-distribution} apply with $\A$ replaced by $\alpha\A$.

\paragraph{Eigendecomposition.} Let $u_k(n) = \sqrt{2/(N+1)}\,\sin(\pi n k/(N+1))$, $n=0,\dots,N+1$, so that $u_k(0)=u_k(N+1)=0$. Using $\sin(a-c)+\sin(a+c) = 2\sin a\cos c$,
\[
u_k(n-1) - 2u_k(n) + u_k(n+1) = 2\big(\cos\tfrac{\pi k}{N+1} - 1\big)u_k(n) = -4\sin^2\!\big(\tfrac{\pi k}{2(N+1)}\big)\,u_k(n), \quad n = 1,\dots,N,
\]
so $\mathbf{u}_k = (u_k(1),\dots,u_k(N))^\top$ is an eigenvector of $\A$ with eigenvalue $\lambda_k = -4\sin^2(\pi k/(2(N+1)))\in(-4,0)$, and $\mathbf{U} = [\mathbf{u}_1,\dots,\mathbf{u}_N]$ is orthogonal~\citep{strang2007computational}. In particular, $\A$ is negative definite and invertible, as required by Propositions~\ref{thm:correlated-process-score}--\ref{thm:correlated-bridge-distribution}.

\paragraph{Decoupling.} Let $\mathbf{z}_t = \mathbf{U}^\top\x_t$ and $\boldsymbol{\beta} = \mathbf{U}^\top\bv$. Since $\mathbf{U}$ is orthogonal, $\mathbf{w}_t = \mathbf{U}^\top\W_t$ is again a standard Wiener process, and
\[
dz_{k,t} = (\alpha\lambda_k z_{k,t} + \beta_k)\,dt + \sqrt{\epsilon}\,dw_{k,t}, \qquad k=1,\dots,N,
\]
are $N$ independent scalar Ornstein--Uhlenbeck processes. Applying Propositions~\ref{thm:correlated-process-score}--\ref{thm:correlated-bridge-distribution} with $1\times1$ matrices gives, for $0\le t<t'$,
\begin{gather*}
m_k(t) = e^{\alpha\lambda_k t} z_{k,0} + \big(e^{\alpha\lambda_k t}-1\big)\beta_k/(\alpha\lambda_k), \qquad
\sigma_k^2(t) = \epsilon\,\frac{1 - e^{2\alpha\lambda_k t}}{2\alpha|\lambda_k|},\\
\mathbb{E}[z_{k,t}\mid z_{k,0}, z_{k,t'}] = m_k(t) + e^{\alpha\lambda_k (t'-t)}\frac{\sigma_k^2(t)}{\sigma_k^2(t')}\big(z_{k,t'} - m_k(t')\big), \qquad
\mathrm{Var}[z_{k,t}\mid z_{k,0}, z_{k,t'}] = \sigma_k^2(t) - e^{2\alpha\lambda_k (t'-t)}\frac{\sigma_k^4(t)}{\sigma_k^2(t')},
\end{gather*}
and the matrix expressions~\eqref{eq:main-mean}--\eqref{eq:bridge-cov} are recovered as $\mathbf{U}\,\mathrm{diag}(\cdot)\,\mathbf{U}^\top$ of these scalars. We verified numerically that the two forms coincide.

\paragraph{(i) Mean.} With $\mathbf{m}^* = -(\alpha\A)^{-1}\bv$, we have $(e^{\alpha\A t}-\I)(\alpha\A)^{-1}\bv = (\I - e^{\alpha\A t})\mathbf{m}^*$, hence $\muvec_{t|0} = e^{\alpha\A t}\X_0 + (\I - e^{\alpha\A t})\mathbf{m}^*$. Frequency $k$ of $\X_0$ is multiplied by $e^{-\alpha|\lambda_k| t}$ and replaced by the corresponding frequency of $\mathbf{m}^*$; since $|\lambda_k|$ increases with $k$, high temporal frequencies are damped first (Figure~\ref{fig:spectral}, left).

\paragraph{(ii) Variance and stationary covariance.} $\sigma_k^2(t)\le\epsilon t$, with $\sigma_k^2(t) = \epsilon t\,(1 + O(\alpha|\lambda_k|t))$ for small $\alpha|\lambda_k| t$, and $\sigma_k^2(t)\to\epsilon/(2\alpha|\lambda_k|)$ as $t\to\infty$ (Figure~\ref{fig:spectral}, middle). The stationary covariance is $\Si_{\infty} = -\tfrac{\epsilon}{2\alpha}\A^{-1}$. Let $G_{ij} = \min(i,j)\,(N+1-\max(i,j))/(N+1)$. For fixed $j$, $g(i) = G_{ij}$ is linear in $i$ on $\{0,\dots,j\}$ and on $\{j,\dots,N+1\}$, with $g(0) = g(N+1) = 0$, so $g(i-1)-2g(i)+g(i+1) = 0$ for $i\neq j$, and for $i=j$
\[
\tfrac{1}{N+1}\big[(j-1)(N+1-j) - 2j(N+1-j) + j(N-j)\big] = -1.
\]
Hence $\A\mathbf{G} = -\I$, i.e., $-\A^{-1} = \mathbf{G}$, which is the covariance of a discrete Brownian bridge (a Gaussian random walk over frames $0,\dots,N+1$ pinned to zero at both ends; Figure~\ref{fig:spectral}, right).

\paragraph{(iii) Boundary-determined profile.} For $\mathbf{m} = \alpha\mathbf{m}^*$, $\A\mathbf{m} = -\bv$ reads $m_{n-1} - 2m_n + m_{n+1} = 0$, $n=1,\dots,N$, with $m_0$ and $m_{N+1}$ given by the boundary values of Section~\ref{sec:theory_tasks}. Its solution is linear in $n$: $\x^0 + \tfrac{n}{N+1}(\x^{N+1}-\x^0)$ for frame interpolation, $\tfrac{N+1-n}{N+1}\x^0$ for image-to-video generation, and $0$ for super-resolution.

\paragraph{Limit $\alpha\to0$.} Since $e^{\alpha\lambda_k t}\to1$, $(e^{\alpha\lambda_k t}-1)/(\alpha\lambda_k) \to t$ and $\sigma_k^2(t)\to\epsilon t$, we get $m_k(t)\to z_{k,0}+t\beta_k$; the constant drift $\beta_k$ cancels in the bridge, and the bridge mean and variance converge to $z_{k,0} + \tfrac{t}{t'}(z_{k,t'}-z_{k,0})$ and $\epsilon\,t(1-t/t')$, i.e., to the Brownian bridge used by standard Bridge Matching, independently for every frequency and therefore for every frame. We confirmed this limit numerically.

\begin{figure}[H]
  \centering
  \includegraphics[width=\linewidth]{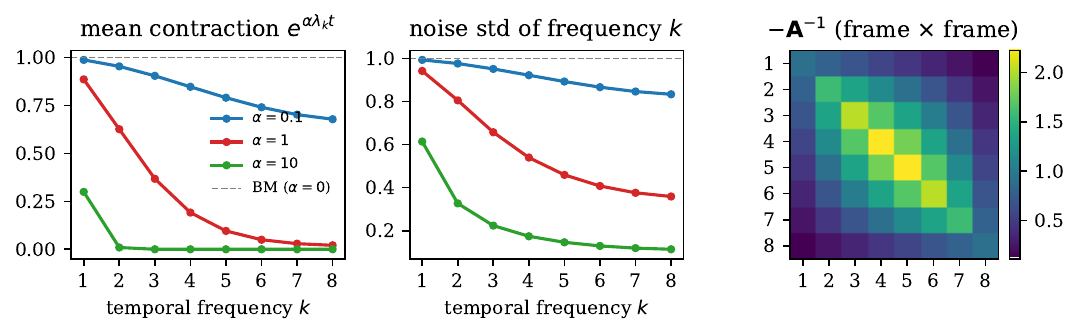}
  \caption{Spectral view of the prior for $N=8$, $\epsilon=1$, $t=1$. Left: contraction $e^{\alpha\lambda_k t}$ of temporal frequency $k$ of $\X_0$ in the prior mean. Middle: standard deviation $\sigma_k(t)$ of the injected noise per frequency; BM ($\alpha\to0$) corresponds to the dashed lines. Right: $-\A^{-1}$, proportional to the stationary covariance between frames (the covariance of a discrete Brownian bridge).}
  \label{fig:spectral}
\end{figure}

\section{Implementation Details}\label{sec:mnist_implementation}

\paragraph{Moving MNIST.}

For the proof-of-concept experiments, we use our own small U-Net based on several residual blocks with $2D$ convolutions, with approximately $8.7$ million parameters. For each experiment, we extract sub-sequences of $L = 10$ consecutive frames from the videos and concatenate them into a $10$-channel input for the network. We use $9{,}500$ training sequences and $500$ validation sequences, with a random split for each seed value. We train each model for $150{,}000$ iterations with a batch size of $128$ and an EMA rate of $0.999$. We use the AdamW optimizer~\citep{loshchilov2018decoupled} with betas $(0.9, 0.95)$, a weight decay of $10^{-4}$, and a learning rate of $3 \times 10^{-5}$. All experiments are conducted on a single NVIDIA A100 GPU. The DDPM and DDIM baselines generate the unknown frames from Gaussian noise in the reverse process, conditioned on the known frames through the network input. For the bridge-based methods, the posterior sampler of Algorithm~\ref{alg:inference} uses a uniform time grid on $[0,1]$; the known frames are reset after every step, and the last step returns the mean of the bridge posterior; we use $K=1000$ steps. All metrics are computed in pixel space $[0,1]$, and FVD uses the I3D network and implementation of StyleGAN-V~\citep{skorokhodov2022styleganv}.

\paragraph{Large-scale frame interpolation.} The pretrained CogVideoX-2B transformer is fine-tuned to predict the clean latents $\widehat{\X}_0$ directly with the regression loss~\eqref{eq:final-obj}, using an empty text conditioning and mapping the bridge time $t\in[0,1]$ to the timestep index of the pretrained scheduler. The prior uses $\alpha = 1$. The first and the last frames are encoded separately by the 3D VAE and enter $\bv$, while the intermediate frames are encoded jointly, and the bridge acts on the resulting latent frames. We use $\epsilon=0.1$ for both bridges and $K=100$ sampling steps. The intermediate frames of $\X_T$ are the VAE latents of Gaussian pixel noise. Since the 3D VAE compresses time by a factor of 4, the key-frame latents cover one frame each, while each intermediate latent covers several frames; the boundary and spectral interpretation of Section~\ref{sec:spectral} therefore transfers to the latent frames only heuristically. RIFE and LDMVFI are retrained on the same data with their official training code and default settings. All metrics (Tables~\ref{tab:interpolation_results_ucf} and~\ref{tab:vimeo_results_intepolation}) are computed with the same script for all methods on the full test splits, over all frames of each clip, against the VAE reconstruction of the ground truth.

\paragraph{Large-scale image-to-video generation.} We use the official FrameBridge implementation with the Latte-S/2 backbone on frame-wise latents and keep its data pipeline, network, optimizer and evaluation protocol; TCVBM differs from FrameBridge only in the bridge, i.e., the time-correlated prior with the image-to-video boundary condition of Section~\ref{sec:theory_tasks} and $\alpha=1$. Both models are trained with batch size 20 at $256\times256$ for 400k iterations.

\paragraph{Toy experiments.} Additional toy experiments on two-dimensional Gaussian and Swiss-roll data are included in our code. They compare the closed-form bridge of Proposition~\ref{thm:correlated-bridge-distribution} with an Euler--Maruyama simulation of the prior SDE and compare TCVBM with BM when generating the intermediate frames of a rotating Swiss roll from Gaussian noise.

\section{Additional Quantitative Results}\label{appendix:metrics}

This appendix reports the standard deviations of Table~\ref{tab:moving_mnist_results} over five seeds (Appendix~\ref{appendix:confidence}) and a user study on Moving MNIST (Appendix~\ref{sec:mnist_human}).

\subsection{Standard Deviations over Seeds}\label{appendix:confidence}

Here, we present the results of a quantitative comparison between DDPM, DDIM, Bridge Matching (BM), and our proposed method, TCVBM. We report the mean $\pm$ standard deviation over 5 runs of each method with different random seeds (and different random train/validation splits): Tables~\ref{tab:std_fi}, \ref{tab:std_i2v} and~\ref{tab:std_sr} correspond to frame interpolation, image-to-video generation and video super-resolution, respectively.

\begin{table}[H]
  \centering
  \caption{Frame interpolation quantitative results with standard deviation.\label{tab:std_fi} The best values in each column are in bold, the second best are underlined.}
    \begin{tabular}{lcccc}
    \hline
    \textbf{Method} &
    \textbf{FVD $\downarrow$} &
    \textbf{LPIPS $\downarrow$} &
    \textbf{PSNR $\uparrow$} &
    \textbf{SSIM $\uparrow$} \\
    \hline
    DDIM  & 33.61 $\pm$ 5.80 &  0.105 $\pm$ 0.070 & 15.80 $\pm$ 0.12 & 0.760 $\pm$  0.011\\
    DDPM    & \underline{32.40 $\pm$ 1.49} & 0.117 $\pm$ 0.009 & 14.94 $\pm$ 0.43 &  0.741 $\pm$ 0.024\\
    BM  & 34.32 $\pm$ 0.39 &  \underline{0.079 $\pm$ 0.001} & \underline{17.10 $\pm$ 0.39} &  \underline{0.794 $\pm$ 0.005}\\
    TCVBM (ours) & \textbf{30.54 $\pm$ 4.03} & \textbf{0.077 $\pm$ 0.019} & \textbf{17.28 $\pm$ 0.46} & \textbf{0.813 $\pm$ 0.044} \\
    \hline
    \end{tabular}
\end{table}

\begin{table}[H]
  \centering
  \caption{Image-to-video generation quantitative results with standard deviation.\label{tab:std_i2v} The best values in each column are in bold, the second best are underlined.}
    \begin{tabular}{lcccc}
    \hline
    \textbf{Method} &
    \textbf{FVD $\downarrow$} &
    \textbf{LPIPS $\downarrow$} &
    \textbf{PSNR $\uparrow$} &
    \textbf{SSIM $\uparrow$} \\
    \hline
    DDIM  & 77.72 $\pm$ 24.11 & 0.294 $\pm$ 0.092 & \textbf{10.88 $\pm$ 0.51} & \textbf{0.603 $\pm$ 0.069}\\
    DDPM  & 75.86 $\pm$ 13.49 & 0.311 $\pm$ 0.054 & 10.72 $\pm$ 0.27 & \underline{0.595 $\pm$ 0.046} \\
    BM  &  \underline{49.32 $\pm$ 0.56} & \underline{0.271 $\pm$ 0.005} & {10.63 $\pm$ 0.05} & {0.579 $\pm$ 0.004}\\
    TCVBM (ours)  & \textbf{44.96 $\pm$ 0.91}  & \textbf{0.258 $\pm$ 0.002} & \underline{10.75 $\pm$ 0.03} & 0.591 $\pm$ 0.001 \\
    \hline
    \end{tabular}
\end{table}

\begin{table}[H]
  \centering
  \caption{Video super-resolution quantitative results with standard deviation.\label{tab:std_sr} The best values in each column are in bold, the second best are underlined.}
    \begin{tabular}{lcccc}
    \hline
    \textbf{Method} &
    \textbf{FVD $\downarrow$} &
    \textbf{LPIPS $\downarrow$} &
    \textbf{PSNR $\uparrow$} &
    \textbf{SSIM $\uparrow$} \\
    \hline
    DDIM  & 334.70 $\pm$ 5.18 & 0.514 $\pm$ 0.006 & 17.31 $\pm$ 0.10 & 0.613 $\pm$ 0.016 \\
    DDPM & 607.88 $\pm$ 6.28 & 0.236 $\pm$ 0.001 & 20.19 $\pm$ 0.05 & 0.582 $\pm$ 0.004\\
    BM  &  \textbf{53.52 $\pm$ 20.68} &  \textbf{0.020 $\pm$ 0.005} & \underline{22.64 $\pm$ 1.04} &  \underline{0.954 $\pm$ 0.012}\\
    TCVBM (ours)  & \underline{59.49 $\pm$ 23.15} & \textbf{0.020 $\pm$ 0.004} & \textbf{22.67 $\pm$ 1.42} & \textbf{0.970 $\pm$ 0.011} \\
    \hline
    \end{tabular}
\end{table}

\subsection{Human Evaluation on Moving MNIST}\label{sec:mnist_human}

We conducted an additional user study on 40 randomly generated examples using initial frames from the Moving MNIST test dataset. The results are presented in Table \ref{tab:human_mnist}.

\begin{table}[h]
  \centering
  \footnotesize
  \caption{Human evaluation on the Moving MNIST dataset: share of votes (\%) preferring each method or both.}
  \setlength{\tabcolsep}{3pt}
    \resizebox{\textwidth}{!}{%
    \begin{tabular}{l cc cc cc}
    \hline
    Method &
    \multicolumn{2}{c}{Frame interpolation} &
    \multicolumn{2}{c}{Image-to-video generation} &
    \multicolumn{2}{c}{Video super-resolution} \\
    \cmidrule(lr){2-3}\cmidrule(lr){4-5}\cmidrule(lr){6-7} &
    Frame quality &
    Temporal consistency &
    Frame quality &
    Temporal consistency &
    Frame quality &
    Temporal consistency \\
    \hline
    BM  &  17\% & 20.5\% & 28\% & 18.5\% & 11.5\% & 13.5\% \\
    TCVBM (ours) & 36\% & 52.5\% & 51\% & 59.5\% & 19\% & 30.5\% \\
    Both & 47\% & 27\% & 21\% & 22\% & 69.5\% & 56\% \\
    \hline
    \end{tabular}}\label{tab:human_mnist}
\end{table}

\section{Initialization Experiments}\label{appendix:initialization}

In this section, we compare options for initializing the input $\X_T$ of the bridge-based methods, BM and TCVBM. Since bridges start generation from $\X_T$ rather than from pure noise, this choice can affect the results.

\subsection{Frame Interpolation}\label{appendix:initialization_interpolation}

As input data for the network, we explored two options: filling in intermediate frames with Gaussian noise sampled from $\mathcal{N}(\textbf{0}, \I)$ and using linear interpolation between fixed boundary frames $\x^0$ and $\x^N$, i.e.:
\[\x^n_{input} = \frac{(N - n)\x^0 + n\x^N}{N},\quad n = 1, \dots, N-1.\]
Table~\ref{tab:init_interpolation} compares the results of these initialization methods. As can be seen, filling intermediate frames with noise from a normal distribution produces better results than the initial linear interpolation.

\begin{table}[H]
  \centering
  \caption{Analysis of the impact of initialization of input video data for bridge-based methods in the task of frame interpolation. Using noise from a normal distribution shows a clear advantage. The best values in each column are in bold, the second best are underlined.}
    \begin{tabular}{llcccc}
    \hline
    \textbf{Initialization method} & \textbf{Method} &
    \textbf{FVD $\downarrow$} &
    \textbf{LPIPS $\downarrow$} &
    \textbf{PSNR $\uparrow$} &
    \textbf{SSIM $\uparrow$} \\
    \hline
    Linear interpolation & BM  & 34.80 & 0.109 & 15.44 & 0.756 \\
    & TCVBM (ours) & \underline{31.94} & 0.092 & 16.28 & 0.782 \\
    \hline
    Gaussian noise from $\mathcal{N}(\textbf{0}, \I)$ & BM & 34.32 & \underline{0.079} & \underline{17.10} & \underline{0.794} \\
    & TCVBM (ours) & \textbf{30.54} & \textbf{0.077} & \textbf{17.28} & \textbf{0.813}\\
    \hline
    \end{tabular}\label{tab:init_interpolation}
\end{table}

\subsection{Image-to-Video Generation}\label{appendix:initialization_i2v}

Here we compare the following two types of initialization: duplicating the first frame in place of the frames to be generated (static video) and using random noise everywhere except the first frame. Static video initialization is superior to the noise option for both models (Table~\ref{tab:initialization_i2v}).

\begin{table}[H]
  \centering
  \caption{Comparison of two types of initialization of input data for image-to-video generation. The best values in each column are in bold, the second best are underlined.}
    \begin{tabular}{llcccc}
    \hline
    \textbf{Initialization method} & \textbf{Method} &
    \textbf{FVD $\downarrow$} &
    \textbf{LPIPS $\downarrow$} &
    \textbf{PSNR $\uparrow$} &
    \textbf{SSIM $\uparrow$} \\
    \hline
    Static video & BM & 49.32 & 0.271 & 10.63 & 0.579\\
    & TCVBM (ours) & \textbf{44.96} & \textbf{0.258} & \textbf{10.75} & \textbf{0.591}\\
    \hline
    Gaussian noise from $\mathcal{N}(\textbf{0}, \I)$& BM & 52.57 & 0.287 & 10.61 & 0.568 \\
    & TCVBM (ours) & \underline{48.61} & \underline{0.263} & \underline{10.68} & \underline{0.587}\\
    \hline
    \end{tabular}\label{tab:initialization_i2v}
\end{table}

\subsection{Video Super-Resolution}\label{appendix:initialization_sr}

For super-resolution, we compare feeding the upsampled low-resolution video as $\X_T$ with concatenating it with Gaussian noise (Table~\ref{tab:init_sr}); both options give similar results.

\begin{table}[H]
  \centering
  \caption{Comparison of two types of initialization of input data for video super-resolution.\label{tab:init_sr} We perform this comparison for low-resolution $32\times32$. The best values in each column are in bold, the second best are underlined.}
    \begin{tabular}{llcccc}
    \hline
    \textbf{Initialization method} & \textbf{Method} &
    \textbf{FVD $\downarrow$} &
    \textbf{LPIPS $\downarrow$} &
    \textbf{PSNR $\uparrow$} &
    \textbf{SSIM $\uparrow$} \\
    \hline
    Low-resolution video & BM & 60.69 & \underline{0.014} & 24.888 & 0.972\\
    & TCVBM (ours) & \textbf{58.26} & \textbf{0.012} & \underline{24.970} & \textbf{0.973}\\
    \hline
   Low-resolution video  & BM & 60.56 & 0.022 & 24.892 & \textbf{0.973}\\
   concatenated with noise from $\mathcal{N}(\textbf{0}, \I)$ & TCVBM (ours) & \underline{59.65} & 0.022 & \textbf{24.988} & \textbf{0.973}\\
    \hline
    \end{tabular}
\end{table}

\section{Hyperparameter Study}\label{appendix:hyperparameters}

Tables~\ref{tab:hyper_interpolation}, \ref{tab:hyper_i2v} and~\ref{tab:hyper_superres} report TCVBM on Moving MNIST for a grid of $\epsilon$ and $\alpha$. Here, $\alpha$ multiplies $\A$ while $\bv$ is kept fixed (Section~\ref{sec:theory_tasks}), so $\alpha$ changes both the coupling strength and the scale $1/\alpha$ of the boundary-determined profile.

Here we compare different values of hyperparameters, namely the noise scaling value $\epsilon$ and the coefficient $\alpha$, which scales the matrix $\A$. Tables~\ref{tab:hyper_interpolation} and ~\ref{tab:hyper_i2v} show that in the case of frame interpolation and image-to-video generation, it is impossible to identify a clear dependence of the generation quality on the hyperparameters used, however, a sufficient amount of noise and not large values for the $\alpha$ coefficient are optimal. The results for video super-resolution in Table~\ref{tab:hyper_superres} demonstrate that small values of $\epsilon$ and $\alpha$ are optimal for this task, which does not contradict the results for frame interpolation.

\begin{table}[H]
    \centering
    \caption{\small The results of TCVBM (ours) training with various hyperparameters $\epsilon$ and $\alpha$ for frame interpolation. The best values in each column are in bold, the second best are underlined.}
    \begin{tabular}{llcccc}
    \hline
    \textbf{$\epsilon$} &
    \textbf{$\alpha$} &
    \textbf{FVD $\downarrow$} &
    \textbf{LPIPS $\downarrow$} &
    \textbf{PSNR $\uparrow$} &
    \textbf{SSIM $\uparrow$} \\
    \hline
    0.1 & 0.1 & 35.57 & 0.085 & 16.40 & 0.797\\
    0.1 & 1 & 36.54 & 0.089 & 16.37 & 0.797\\
    0.1 & 10 & \underline{29.79} & 0.086 & 16.56 & 0.801\\
    1 & 0.1 & \textbf{27.34} & 0.084 & 16.65 & 0.803\\
    1 & 1 & 30.54 & \textbf{0.077} & \textbf{17.28} & 0.813\\
    1 & 10 & 31.43 & \underline{0.080} & 17.12 & 0.817\\
    10 & 0.1 & 37.66 & 0.084 & \underline{17.24} & \textbf{0.819}\\
    10 & 1 & 54.54 & 0.093 & {17.17} & \underline{0.818}\\
    10 & 10 & 54.56 & 0.100 & {17.17} & 0.809\\
    \hline
\end{tabular}\label{tab:hyper_interpolation}
\end{table}

\begin{table}[h!]
    \centering
    \caption{The results of TCVBM (ours) training with various hyperparameters $\epsilon$ and $\alpha$ for image-to-video generation. The best values in each column are in bold, the second best are underlined.}
    \begin{tabular}{llcccc}
            \hline
            \textbf{$\epsilon$} &
            \textbf{$\alpha$} &
            \textbf{FVD $\downarrow$} &
            \textbf{LPIPS $\downarrow$} &
            \textbf{PSNR $\uparrow$} &
            \textbf{SSIM $\uparrow$} \\
            \hline
            0.1 & 0.1 &  55.46	&  0.2670 & \textbf{10.80} & 0.587 \\
            0.1 & 1 &  51.62 & 0.2578 & 10.74 & 0.583 \\
            0.1 & 10 & 59.57 & \textbf{0.2571} & 10.67 & 0.585 \\
            1 & 0.1 & 51.80 & 0.2615 & 10.70 & 0.590 \\
            1 & 1 & 44.96 & \underline{0.2575} & \underline{10.75} & \textbf{0.591} \\
            1 & 10 & 51.27 & 0.2613 & 10.60 & 0.583 \\
            10 & 0.1 & 44.90 & 0.2610 & 10.67 & 0.588 \\
            10 & 1 & \textbf{44.22} & 0.2584 & 10.70 & \textbf{0.591} \\
            10 & 10 & \underline{44.58} & 0.2597 & 10.58 & 0.583 \\
            \hline
    \end{tabular}\label{tab:hyper_i2v}
\end{table}

\begin{table}[H]
  \centering
  \caption{\small The results of TCVBM (ours) training with various hyperparameters $\epsilon$ and $\alpha$ for video super-resolution. We perform this comparison for low-resolution $16\times16$. The best values in each column are in bold, the second best are underlined.}
  \begin{tabular}{llcccc}
    \hline
    \textbf{$\epsilon$} &
    \textbf{$\alpha$} &
    \textbf{FVD $\downarrow$} &
    \textbf{LPIPS $\downarrow$} &
    \textbf{PSNR $\uparrow$} &
    \textbf{SSIM $\uparrow$} \\
    \hline
    0.1 & 0.1 & \textbf{59.49} & \textbf{0.020} & \textbf{22.67}	& \textbf{0.970}\\
    0.1 & 1 & \underline{60.41} & \underline{0.023} & \underline{22.35} & \underline{0.968}\\
    1 & 0.1 & 63.23 & 0.029 & 21.00 & 0.957\\
    1 & 1 & 65.02 & 0.032 & 20.12 & 0.949\\
    10 & 0.1 & 68.46 & 0.033 & 20.08 & 0.921\\
    10 & 1 & 70.75 & 0.039 & 19.28 & 0.906\\
    \hline
    \end{tabular}\label{tab:hyper_superres}
\end{table}

\section{Computational Cost Analysis}\label{appendix:comp_cost}

As described in Algorithm~\ref{alg:inference}, inference consists of two alternating steps:
        (i) computing $\widehat{\mathbf{X}}_{0}^{\phi}(\mathbf{X}_{t_k}, t_k)$ and
        (ii) sampling from the Gaussian distribution~\eqref{eq:bridge-distr} with $t = t_{k-1}$, $t' = t_k$ and $\X_0$ replaced by $\widehat{\X}_0$.
        Sampling can be done using reparameterization:
        $$\mathbf{X}_{t} = \muvec_{t|0, t'} + \Si_{t|0, t'}^{1/2}\mathbf{Z}, \quad \mathbf{Z} \sim \mathcal{N}(0, \I),
        $$
        which requires the evaluation of the mean and covariance (Proposition~\ref{thm:correlated-bridge-distribution}):
        $$
        \muvec_{t|0, t'} = \muvec_{t|0}(\X_0) + \mathbf{G}_{t,t'} (\X_{t'} - \muvec_{t'|0}(\X_0)), \qquad
        \Si_{t|0, t'} = \Si_{t|0} - \mathbf{G}_{t,t'}\, e^{\A(t'-t)}\,\Si_{t|0}, \qquad \mathbf{G}_{t,t'} := e^{\A(t'-t)}\Si_{t|0} \Si_{t'|0}^{-1}.
        $$
        In turn, 
        $$
            \muvec_{t|0}(\X_0) = e^{\A t} \X_0 + \left(e^{\A t} - \I\right)\A^{-1}\bv, \quad
            \Si_{t|0} = \epsilon \frac{e^{2\A t} - \I}{2} \A^{-1}.
        $$
        All $N\times N$ matrices above depend only on $\A$ and the time grid, so they can be precomputed once (e.g., from the eigendecomposition of $\A$, Appendix~\ref{appendix:spectral}) and cached. What remains are 5 matrix multiplications per step: $\Si_{t|0, t'}^{1/2}\mathbf{Z}$, $e^{\A t}\widehat{\X}_0$, $\mathbf{G}_{t,t'} \X_{t'}$, $e^{\A t'}\widehat{\X}_0$, $\mathbf{G}_{t,t'} (e^{\A t'}\widehat{\X}_0)$ (the terms with $\bv$ are handled in the same way).
        During training, sampling $\X_t \sim q(\X_t|\X_0,\X_T)$ at a random $t$ requires the same number of products, with the $N\times N$ matrices obtained from the precomputed eigendecomposition of $\A$ at cost $O(N^2)$ per sample.
        All these multiplications are for tensors of size $F \times F$ ($F = N$ is the number of frames in the video) with a tensor of size $F \times C \times H \times W$, where $C$ is the number of channels, $H$ is the height, and $W$ is the width of the video. This requires $O(F^2 C H W)$ operations. For comparison, a single convolutional layer with $C$ input channels, $C_{\text{out}}$ output channels, and kernel size $K \times K$ applied to all $F$ frames requires $O(F C H W K^2 C_{\text{out}})$ operations. Since typically $F \sim 10{-}100$, $K^2 \sim 10$, and $C_{\text{out}} \sim 10$, we have
        $$
            \frac{F^2 C H W}{F C H W K^2 C_{\text{out}}}
            = \frac{F}{K^2 C_{\text{out}}}
            \sim 0.1 - 1.
        $$
        Thus, by operation count, each of these 5 operations requires no more computation than a single convolution layer. Since a neural network requires many convolutional layers and additional nonlinear processing, the overhead of our bridge update is comparable to only a few convolutional layers.

\section{Dynamical Correlation}\label{appendix:dynamical_correlation}

\subsection{Theory}
Consider the prior SDE with an additional multiplicative function $f(t)$ depending only on time $t$:
$$
    d\X_t = f(t)(\A \X_t + \bv)dt + \sqrt{\epsilon}d\W_t.
$$
Unlike the constant $\alpha$ of Section~\ref{sec:theory_tasks}, $f(t)$ multiplies both $\A$ and $\bv$, so the boundary-determined profile does not depend on $f$. The derivation of formulas for this prior follows the same principles used for $f(t) = 1$ in Appendix~\ref{appendix:proofs}.

Define:
$$
    F(t)=\int_0^t f(\tau)\,d\tau,\qquad \bPhi(t)=e^{\A\,F(t)}.
$$
Then $\frac{d}{dt}\bPhi(t)=f(t)\,\A\,\bPhi(t)$, $\bPhi(0)=I$, and $\frac{d}{dt}\bPhi(t)^{-1}=-\,f(t)\,\bPhi(t)^{-1}\A$.

\textbf{Conditional mean.} Consider $\Y_t:=\bPhi(t)^{-1}\X_t$. By Itô’s rule:
\begin{gather}
    d\Y_t = \bPhi(t)^{-1} d\X_t + d(\bPhi(t)^{-1})\,\X_t 
      = \bPhi(t)^{-1}f(t)(\A \X_t+ \bv)\,dt + \sqrt{\epsilon}\,\bPhi(t)^{-1}d\W_t - f(t)\,\bPhi(t)^{-1}\A \X_t\,dt =
      \nonumber
      \\
      = \bPhi(t)^{-1} f(t) \bv\,dt + \sqrt{\epsilon}\,\bPhi(t)^{-1} d\W_t.
      \nonumber
\end{gather}
Integrating from $0$ to $t$ gives
$$
    \Y_t = \X_0 + \int_0^t \bPhi(s)^{-1} f(s)\,\bv\,ds + \sqrt{\epsilon}\int_0^t \bPhi(s)^{-1} d\W_s,
$$
and thus
$$
    \X_t = \bPhi(t)\X_0 + \bPhi(t)\!\int_0^t \bPhi(s)^{-1} f(s)\,\bv\,ds + \sqrt{\epsilon}\,\bPhi(t)\!\int_0^t \bPhi(s)^{-1}d\W_s.
$$
Taking expectation and using $\bPhi(t)\bPhi(s)^{-1}=e^{\A(F(t)-F(s))}$,
$$
    \muvec_{t|0}(\X_0)
    =\mathbb{E}[\X_t|\X_0]
    = e^{\A F(t)}\X_0 + \int_0^t e^{\A(F(t)-F(s))} f(s)\,\bv\,ds.
$$
With the change of variables $u=F(s)$ (so $du=f(s)\,ds$) this equals
$$
\int_{0}^{F(t)} e^{\A(F(t)-u)}\,du\,\bv
=\Big[-\,e^{\A(F(t)-u)}\A^{-1}\Big]_{u=0}^{F(t)} \bv
= \big(e^{\A F(t)}-I\big)\A^{-1}\bv.
$$
Therefore
$$
\muvec_{t|0}(\X_0) = e^{\A F(t)}\X_0 + \big(e^{\A F(t)}-I\big)\A^{-1}\bv .
$$

\textbf{Conditional variance.} 
Let $\Z_t:=\X_t-\muvec_{t|0}$ be the centered process. Subtracting the mean SDE from the original SDE yields
$$
    d\Z_t = f(t)\,\A\,\Z_t\,dt + \sqrt{\epsilon}\,d\W_t.
$$
Define the covariance $\Si_{t|0}:=\mathbb{E}[\Z_t \Z_t^\top]$. Using Itô’s rule for $\Z_t\Z_t^\top$,
$$
    d(\Z_t \Z_t^\top) = (d\Z_t)\Z_t^\top + \Z_t(d\Z_t)^\top + d\Z_t(d\Z_t)^\top,
$$
where
\begin{gather}
(d\Z_t)\,\Z_t^\top
= f(t)\A \Z_t \Z_t^\top\,dt + \sqrt{\epsilon}\,d\W_t\,\Z_t^\top,\\
\Z_t\,(d\Z_t)^\top
= f(t)\Z_t \Z_t^\top \A^\top\,dt + \sqrt{\epsilon}\,\Z_t\,d\W_t^\top,\\
(d\Z_t)(d\Z_t)^\top
= \epsilon\, d\W_t d\W_t^\top = \epsilon\, I\,dt.
\end{gather}

Hence
$$
    d(\Z_t \Z_t^\top)
    = f(t)\big(\A \Z_t \Z_t^\top + \Z_t \Z_t^\top \A^\top\big)\,dt
    + \sqrt{\epsilon}\,d\W_t\,\Z_t^\top
    + \sqrt{\epsilon}\,\Z_t\,d\W_t^\top
    + \epsilon\,I\,dt.
$$

and taking expectations gives
$$
    \frac{d}{dt}\Si_{t|0}
    = f(t)\,\A\,\Si_{t|0} + f(t)\,\Si_{t|0}\A^\top + \epsilon\,\I, \quad \Si_{0|0} = 0.
$$
Its solution is $\Si_{t|0} = \epsilon\int_0^t e^{2\A(F(t)-F(s))}\,ds = \epsilon\,\mathbf{U}\,\mathrm{diag}\big(\int_0^t e^{2\lambda_k(F(t)-F(s))}ds\big)_{k=1}^N\,\mathbf{U}^\top$, where $\A=\mathbf{U}\boldsymbol{\Lambda}\mathbf{U}^\top$ (Appendix~\ref{appendix:spectral}); for $f(t)\equiv1$ it reduces to~\eqref{eq:main-cov}, and for a general $f$ the $N$ scalar integrals are computed numerically.
To get the cross-covariance, we use:
\[
\Z_t = \sqrt{\epsilon}\int_0^t e^{\A(F(t)-F(s))}\,d\W_s,\qquad
\Z_{t'} = \sqrt{\epsilon}\int_0^{t'} e^{\A(F(t')-F(u))}\,d\W_u,
\]
The Itô isometry yields
\[
\Si_{t,t'|0}
:= \mathrm{Cov}(\Z_t,\Z_{t'}^\top)
= \epsilon\int_0^{t} e^{\A(F(t)-F(s))}\,e^{\A(F(t')-F(s))^\top}\,ds.
\]
For symmetric $\A$, $e^{(\cdot)^\top}=e^{(\cdot)}$ and, since these exponentials commute (all are $e^{\A(\cdot)}$),
\[
e^{\A(F(t)-F(s))}\,e^{\A(F(t')-F(s))}
= e^{\A(F(t')-F(t))}\,e^{2\A(F(t)-F(s))}.
\]
Hence
\[
\Si_{t,t'|0}
= e^{\A(F(t')-F(t))}\,\epsilon\int_0^t e^{2\A(F(t)-F(s))}\,ds
= e^{\A\,[F(t')-F(t)]}\,\Si_{t|0} .
\]

\textbf{Summary.}
Thus, all three components: mean $\muvec_{t|0}(\X_0)$, variance $\Si_{t|0}$, and cross-covariance $\Si_{t, t'|0}$ are derived and can be used further in the same way as in the original case of $f(t) = 1$.

\subsection{Experimental Results}

A decreasing function $f(t)$ increases the effective coupling $f(t)\A$ in the reverse process as $t \to 0$. Thus, the correlation of frames with each other in the generated video increases in the last steps of the inference. We conduct a series of experiments to investigate the effect of dynamic correlation and the choice of the function $f(t)$. We use the same experimental setting for Moving MNIST dataset as described in Appendix~\ref{sec:mnist_implementation}, with a number of optimizer steps set to $120{,}000$. The frame interpolation results are presented in Table~\ref{tab:dynamic_correlation}. Note that $f(t)=1-2t$ changes sign at $t=1/2$, so for $t>1/2$ the drift is anti-smoothing; we keep this variant as an extreme case. As can be seen, our experiments do not demonstrate the advantages of using dynamic correlation compared to using constant values of the matrix $\A$. However, we observe significant differences in the quality of the results depending on the function $f(t)$. This demonstrates a complex structure of the relationship between video frames in the diffusion process, which our framework provides in the constant linear approximation.

\begin{table}[H]
  \centering
  \caption{The results of the selection of the function $f(t)$, which determines the inverse dependence of the values of the matrix $\A$ on time $t$, for TCVBM (ours). The best values in each column are in bold, the second best are underlined.}
  \small
    \begin{tabular}{lcccc}
    \hline
    $f(t)$ &
    \textbf{FVD $\downarrow$} &
    \textbf{LPIPS $\downarrow$} &
    \textbf{PSNR $\uparrow$} &
    \textbf{SSIM $\uparrow$} \\
    \hline
    $1 - t$ & 49.569 & 0.115 & \underline{13.938} & \underline{0.752} \\
    $1 - 2t$ & 398.823 & 0.202 & 11.760 & 0.684 \\
    $1 - 0.5t$ & 427.942	& 0.226 & 12.063 & 0.620 \\
    $2\times(1 - t)$ & 409.495 & 0.195 & 12.239 & 0.624 \\
    $0.5\times(1 - t)$ & 269.899 & 0.182 & 12.297 & 0.676 \\
    $(1 - t)^2$ & \underline{43.651} & \underline{0.112} & 13.916 & 0.751 \\
    $e^{-2t}$ & 1216.250 & 0.629 & 7.243 & 0.119 \\
    $e^{-4t}$ & 108.052 & 0.142 & 13.186 & 0.688 \\
    $e^{-8t}$ & 50.078 & 0.128 & 13.321 & 0.726 \\
    \hline
    Constant ($f(t) = \alpha = 1$) & \textbf{36.309} & \textbf{0.097} & \textbf{14.515} & \textbf{0.772} \\
    \hline
    \end{tabular}\label{tab:dynamic_correlation}
\end{table}

\end{document}